\documentclass[11pt]{article}
\pdfoutput=1

\usepackage{fullpage,times}

\usepackage[utf8]{inputenc}
\usepackage[T1]{fontenc}
\usepackage{hyperref}
\hypersetup{
        colorlinks   = true,
        urlcolor     = blue,
        linkcolor    = blue,
        citecolor   = blue
}

\usepackage{url}
\usepackage{booktabs}
\usepackage{amsmath,amsfonts,amsthm,amssymb}
\usepackage{nicefrac}
\usepackage{microtype}

\usepackage[round]{natbib}

\usepackage{enumitem}
\usepackage{algorithm}
\usepackage{algpseudocode}
\usepackage[nolist,nohyperlinks]{acronym}
\usepackage{color}
\usepackage{bold-extra}
\usepackage{mathtools}
\usepackage{etoolbox}
\usepackage{caption}
\usepackage{subcaption}
\usepackage{authblk}

\renewcommand{\tilde}{\widetilde}
\renewcommand{\P}{\mathbb{P}}

\newcommand{\bA}{\boldsymbol{A}}

\newcommand{\bB}{\boldsymbol{B}}

\newcommand{\bx}{\boldsymbol{x}}

\newcommand{\boldeta}{\boldsymbol{Z}}
\newcommand{\myboldeta}{\boldsymbol{\eta}}

\newcommand{\bX}{\boldsymbol{X}}

\newcommand{\bM}{\boldsymbol{M}}
\newcommand{\bH}{\boldsymbol{H}}

\newcommand{\bI}{\boldsymbol{I}}
\newcommand{\bu}{\boldsymbol{u}}
\newcommand{\by}{\boldsymbol{y}}

\newcommand{\bz}{\boldsymbol{z}}

\newcommand{\bS}{\boldsymbol{S}}

\newcommand{\bw}{\boldsymbol{w}}

\newcommand{\btilw}{\boldsymbol{\tilde{w}}}
\newcommand{\btilu}{\widetilde{\boldsymbol{u}}}

\newcommand{\bhatw}{\wh{\boldsymbol{w}}}
\newcommand{\bwstar}{\boldsymbol{w}^{\star}}
\newcommand{\bxstar}{\boldsymbol{x}^{\star}}

\newcommand{\bxstartp}{\boldsymbol{x}^{\star \top}}

\newcommand{\bV}{\boldsymbol{V}}
\newcommand{\bU}{\boldsymbol{U}}

\newcommand{\btilV}{\boldsymbol{\widetilde{V}}}
\newcommand{\bzero}{\boldsymbol{0}}

\newcommand{\sB}{\mathcal{B}}

\newcommand{\sD}{\mathcal{D}}
\newcommand{\sF}{\mathcal{F}}

\newcommand{\ind}[1]{\mathbb{I}\!\left\{ #1 \right\}}

\newcommand{\mtil}{\widetilde{m}}

\DeclareMathOperator*{\argmax}{arg\,max}

\newcommand{\field}[1]{\mathbb{#1}}

\newcommand{\R}{\field{R}}

\newcommand{\Var}{\mathrm{Var}}

\newcommand{\norm}[1]{\left\|{#1}\right\|}
\newcommand{\diag}[1]{\mbox{\rm diag}\!\left\{{#1}\right\}}

\newcommand{\scO}{\mathcal{O}}

\newcommand{\dt}{\displaystyle}

\newcommand{\wh}{\widehat}
\newcommand{\ve}{\varepsilon}

\newtheorem{lemma}{Lemma}

\newtheorem{theorem}{Theorem}
\newtheorem{corollary}{Corollary}

\newtheorem{proposition}{Proposition}
\newtheorem{definition}{Definition}

\newcommand{\reals}{\mathbb{R}}

\newcommand{\tp}{^{\top}}
\newcommand{\ip}[1]{\left\langle #1 \right\rangle}

\DeclareMathOperator*{\E}{\mathbb{E}}

\newtheoremstyle{named}{}{}{\itshape}{}{\bfseries}{}{.5em}{\thmnote{#3}#1}
\theoremstyle{named}
\newtheorem*{nameddef}{}

\newcommand{\tr}{\mathrm{tr}}

\newcommand{\bSigma}{\boldsymbol{\Sigma}}

\newcommand{\pr}[1]{\left( #1 \right)}
\newcommand{\br}[1]{\left[ #1 \right]}
\newcommand{\cbr}[1]{\left\{ #1 \right\}}

\newcommand{\wt}{\widetilde}

\newcommand{\defOtilde}{\stackrel{\wt{\scO}}{=}}

\newcommand{\ofu}{^{\text{\scshape{so}}}}
\newcommand{\ts}{^{\text{\scshape{ts}}}}
\newcommand{\rls}{^{\text{\scshape{rls}}}}

\newcommand{\optts}{^{\text{\scshape{opt-ts}}}}

\newcommand{\patch}[1]{#1}

\begin{acronym}
\acro{IC}{Improvement Condition}
\acro{RLS}{Regularized Least Squares}
\acro{TL}{Transfer Learning}
\acro{HTL}{Hypothesis Transfer Learning}
\acro{ERM}{Empirical Risk Minimization}
\acro{TEAM}{Target Empirical Accuracy Maximization}
\acro{RKHS}{Reproducing kernel Hilbert space}
\acro{DA}{Domain Adaptation}
\acro{LOO}{Leave-One-Out}
\acro{HP}{High Probability}
\acro{RSS}{Regularized Subset Selection}
\acro{FR}{Forward Regression}
\acro{PSD}{Positive Semi-Definite}
\acro{SGD}{Stochastic Gradient Descent}
\acro{OGD}{Online Gradient Descent}
\acro{EWA}{Exponentially Weighted Average}
\acro{EMD}{Effective Metric Dimension}
\acro{PDE}{Partial Differential Equation}
\acro{SDE}{Stochastic Differential Equation}
\acro{FD}{Frequent Directions}
\acro{OFU}{Optimism in the Face of Uncertainty}
\acro{TS}{Thompson Sampling}
\end{acronym}

\usepackage{xcolor}
\usepackage{xspace}

\title{Efficient Linear Bandits through Matrix Sketching}

\author{Ilja Kuzborskij \thanks{ilja.kuzborskij@gmail.com}}
\author{Leonardo Cella \thanks{leonardocella@gmail.com}}
\author{Nicol\`{o} Cesa-Bianchi \thanks{nicolo.cesa-bianchi@unimi.it}}
\affil{
  Dipartimento di Informatica \& DSRC\\
  Universit\`{a} degli Studi di Milano\\
  20133 Milano, Italy}

\begin{document}

\date{September, 2018}
\maketitle

\begin{abstract}
We prove that two popular linear contextual bandit algorithms, OFUL and Thompson Sampling, can be made efficient using Frequent Directions, a deterministic online sketching technique. More precisely, we show that a sketch of size $m$ allows a $\scO(md)$ update time for both algorithms, as opposed to $\Omega(d^2)$ required by their non-sketched versions in general (where $d$ is the dimension of context vectors).
This computational speedup is accompanied by regret bounds of order $(1+\ve_m)^{3/2}d\sqrt{T}$ for OFUL and of order $\big((1+\ve_m)d\big)^{3/2}\sqrt{T}$ for Thompson Sampling, where $\ve_m$ is bounded by the sum of the tail eigenvalues not covered by the sketch.
In particular, when the selected contexts span a subspace of dimension at most $m$, our algorithms have a regret bound matching that of their slower, non-sketched counterparts.
Experiments on real-world datasets corroborate our theoretical results.
\end{abstract}

\section{Introduction}
The stochastic contextual bandit is a sequential decision-making problem where an agent interacts with an unknown environment in a series of rounds. In each round, the environment reveals a set of feature vectors (called contexts, or actions) to the agent. The agent chooses an action from the revealed set and observes the stochastic reward associated with that action (bandit feedback). The strategy used by the agent for choosing actions based on past observations is called a policy. The goal of the agent is to learn a policy minimizing the regret, defined as the difference between the total reward of the optimal policy (i.e., the policy choosing the action with highest expected reward at each round) and the total reward of the agent's policy.

Contextual bandits are a popular modelling tool in many interactive machine learning tasks. A typical area of application is personalized recommendation, where a recommender system selects a product for a given user from a set of available products (each described by a feature vector) and receives a feedback (purchase or non-purchase) for the selected product.

We focus on the \emph{stochastic linear bandit} model \citep{auer2002using,DaniHK08}, where the set of actions (or decision set) is a finite\footnote{
Note that our regret bounds do not actually depend on the cardinality of the sets $D_t$.  
} set $D_t \subset \reals^d$, and the reward for choosing action $\bx_t\in D_t$ is given by
$
Y_t = \bx_t\tp \bwstar + \eta_t
$
where $\bwstar \in \reals^d$ is a fixed and unknown vector of real coefficients and $\eta_t$ is a zero-mean random variable. The regret in this setting is defined by
\begin{equation}
  \label{eq:regret}
  R_T = \sum_{t=1}^T \bxstartp_t \bwstar - \sum_{t=1}^T \bx_t\tp \bwstar
\end{equation}
where $\bxstar_t = \argmax_{\bx \in D_t} \bx\tp \bwstar$ is the optimal action at round $t$. Bounds on the regret typically apply to any individual sequence of decision sets $D_t$ and depend on quantities arising from the interplay between $\bwstar$, the sequence of decision sets, and the randomness of the rewards. Note that $R_T$ is a random variable because the actions $\bx_t\in D_t$ selected by the policy are functions of the past observed rewards. For this reason, our regret bounds only hold with probability at least $1-\delta$, where $\delta$ is a confidence parameter. By choosing $\delta=T^{-1}$, we can instead bound the expected regret $\E\big[R_T\big]$ by paying only a $\ln T$ extra factor in the bound.

We consider two of the most popular algorithms for stochastic linear bandits: OFUL \citep{abbasi2011improved} and linear Thompson Sampling \citep{agrawal2013thompson} (linear TS for short). While exhibiting good theoretical and empirical performances, both algorithms require $\Omega\big(d^2\big)$ time to update their model after each round. In this work we investigate whether it is possible to significantly reduce this update time while ensuring that the regret remains nicely bounded. 

The quadratic dependence on $d$ is due to the computation of the inverse correlation matrix of past actions (a cubic dependence is avoided because each new inverse is a rank-one perturbation of the previous inverse). The occurrence of this matrix is caused by the linear nature of rewards: to compute their decisions, both algorithms essentially solve a regularized least squares problem at every round.
In order to improve the running time, we sketch the correlation matrix using a specific technique ---Frequent Directions, \citep{ghashami2016frequent}--- that works well in a sequential learning setting. While matrix sketching is a well-known approach \citep{woodruff2014sketching}, to the best of our knowledge this is the first work that applies sketching to linear contextual bandits while providing rigorous performance guarantees.

With a sketch size of $m$, a rank-one update of the correlation matrix takes only time $\scO(m d)$, which is linear in $d$ for a constant sketch size. However, this speed-up comes at a price, as sketching reduces the matrix rank causing a loss of information which ---in turn--- affects the least squares estimates used by the algorithms. Our main technical contribution shows that when OFUL and linear TS are run with a sketched correlation matrix, their regret blows up by a factor which is controlled by the spectral decay of the correlation matrix of selected actions. More precisely, we show that the sketched variant of OFUL, called SOFUL, achieves a regret bounded by
\begin{equation}
  \label{eq:regret_skoful}
  R_T \defOtilde \big(1 + \ve_m\big)^{\frac{3}{2}} \Big(m + d \ln\big(1 + \ve_m\big)\Big) \sqrt{T}
\end{equation}
where $m$ is the sketch size and $\ve_m$ is upper bounded by the spectral tail (sum of the last $d-m+1$ eigenvalues) of the correlation matrix for all $T$ rounds.
In the special case when the selected actions span a number of dimensions equal or smaller than the sketch size, then $\ve_m = 0$ implying a regret of order $m\sqrt{T}$.
Thus, we have a regret bound matching that of the slower, non-sketched counterpart.\footnote{The regret bound of OFUL in~\citep[Theorem 3]{abbasi2011improved} is stated as $\scO(d \sqrt{T})$, however, it can be improved for low-rank problems by using the ``log-det'' formulation of the confidence ellipsoid.}
When the correlation matrix has rank larger than the sketch size, the regret of SOFUL remains small to the extent the spectral tail of the matrix grows slowly with $T$.
In the worst case of a spectrum with heavy tails, SOFUL may incur linear regret. In this respect, sketching is only justified when the computational cost of running OFUL cannot be afforded.
Similarly, we prove that the efficient sketched formulation of linear TS enjoys a regret bound of order
\begin{equation}
R_T \defOtilde \Big(m + d \ln(1 + \ve_m\big)\Big) \big(1 + \ve_m\big)^{\frac{3}{2}} \sqrt{dT}~.
\end{equation}
Once again, for $\ve_m = 0$ our bound is of order $m\sqrt{dT}$,
which matches the regret bound for linear TS.
When the rank of the correlation matrix is larger than the sketch size, the bound for linear TS behaves similarly to the bound for SOFUL.

Finally, we show a problem-dependent regret bound for SOFUL. This bound, which exhibits a logarithmic dependence on $T$, depends on the smallest gap $\Delta$ between the expected reward of the best and the second best action across the $T$ rounds,
\begin{equation}
  \label{eq:regret_skoful_gap}
  R_T \defOtilde \frac{1}{\Delta} \big(1 + \ve_m\big)^3 \Big(m + d \ln\big(1 + \ve_m\big) \Big)^2 (\ln T)^2~.
\end{equation}
When $\ve_m(T) = 0$ this bound is of order $\frac{m^2}{\Delta}(\ln T)^2$
which matches the corresponding bound for OFUL.
Experiments on six real-world datasets support our theoretical results.

\vspace{-2mm}
\paragraph{Additional related work.}
For an introduction to contextual bandits, we refer the reader to the recent monograph of \cite{lattimore2018bandit}. The idea of applying sketching techniques to linear contextual bandits was also investigated by \cite{yu2017cbrap}, where they used random projections to preliminarly draw a random $m$-dimensional subspace which is then used in every round of play. However, the per-step computation time of their algorithm is cubic in $m$ rather than quadratic like ours. Morover, random projection introduces an additive error $\ve$ in the instantaneous regret which becomes of order $m^{-1/2}$ for any value of the confidence parameter $\delta$ bounded away from $1$. A different notion of compression in contextual bandits is explored by~\cite{jun2017scalable}, where they use hashing algorithms to obtain a computation time sublinear in the number $K$ of actions. An application of sketching (including Frequent Directions) to speed up 2nd order algorithms for online learning is studied by \cite{luo2016efficient}, in a RKHS setting by \cite{calandriello2017efficient}, and in stochastic optimization by~\cite{gonen2016solving}.
\vspace{-2mm}

\section{Notation and preliminaries}
Let $\sB(\bz,r) \subset \R^d$ be the Euclidean ball of center $\bz$ and radius $r > 0$ and let $\sB(r) = \sB(\bzero,r)$. Given a positive definite $d\times d$ matrix $\bA$, we define the inner product $\ip{\bx,\bz}_{\bA} = \bx\tp \bA \by$ and the induced norm $\norm{\bx}_{\bA} = \sqrt{\bx\tp A \bx}$, for any $\bx,\bz \in \R^d$. Throughout the paper, we write $f \defOtilde g$ to denote $f = \widetilde{\scO}(g)$. The contextual bandit protocol is described in Algorithm~\ref{alg:cb}.
\begin{algorithm}[H]
\caption{(Contextual Bandit)}
\label{alg:cb}
\begin{algorithmic}[1]
\For{$t = 1,2,\ldots$}
\State Get decision set $D_t \subset \reals^d$
\State Use current policy to select action $\bx_t \in D_t$
\State Observe reward $Y_t \in \reals$
\State Use pair $(\bx_t,Y_t)$ to update the current policy 
\EndFor
\end{algorithmic}
\end{algorithm}

We introduce some standard assumptions for the linear contexual bandit setting. At any round $t=1,2,\ldots$ the decision set $D_t \subset \reals^d$ is finite and such that $\|\bx\| \leq L$ for all $\bx \in D_t$ and for all $t \ge 1$. The noise sequence $\eta_1, \eta_2 \ldots, \eta_T$ is conditionally $R$-subgaussian for some fixed constant $R \geq 0$. Formally, for all $t \ge 1$ and all $\lambda\in\reals$,
$
\E\br{e^{\lambda \eta_t} \,\big|\, \eta_1,\dots,\eta_{t-1}} \leq \exp\pr{ {\lambda^2 R^2}/{2} }
$.
Note that this implies $\E[\eta_t \mid \eta_1,\dots,\eta_{t-1}] = 0$ and $\Var[\eta_t \mid \eta_1,\dots,\eta_{t-1}] \leq R^2$. Finally, we assume that a known upper bound $S$ on $\norm{\bwstar}$ is available.

Both OFUL and Linear TS operate by computing a confidence ellipsoid to which $\bwstar$ belongs with high probability. Let
$
\bX_t = [\bx_1, \ldots, \bx_t]\tp
$
be the $t \times d$ matrix of all actions selected up to round $t$ by an arbitrary policy for linear contextual bandits.
For $\lambda > 0$, define the regularized correlation matrix of actions $\bV_t$ and the regularized least squares (RLS) estimate $\bhatw_t$ as
\begin{equation}
\label{eq:rls}
  \bV_t = \bX_t\tp \bX_t + \lambda \bI \quad\text{and}\quad \bhatw_t = \bV_t^{-1} \sum_{s=1}^t \bx_s Y_s~.
\end{equation}
The following theorem \cite[Theorem 2]{abbasi2011improved} bounds in probability the distance, in terms of the norm $\norm{\cdot}_{\bV_t}$, between the optimal parameter $\bwstar$ and the \acs{RLS} estimate $\bhatw_t$.
\begin{theorem}[Confidence Ellipsoid]
\label{thm:confidence_oful}
Let $\bhatw_t$ be the RLS estimate constructed by an arbitrary policy for linear contextual bandits after $t$ rounds of play. For any $\delta \in (0,1)$, the optimal parameter $\bwstar$ belongs to the set
$C_t \equiv \cbr{ \bw \in \R^d ~:~ \|\bw - \bhatw_t\|_{\bV_t} \leq \beta_t(\delta) }
$
with probability at least $1-\delta$, where
\begin{equation}
\label{eq:beta}
\beta_t(\delta) = R \sqrt{ d \ln\pr{1 + \frac{t L^2}{\lambda d}} + 2 \ln\pr{\frac{1}{\delta}} } + S \sqrt{\lambda}~.
\end{equation}
\end{theorem}
\paragraph{OFUL.} The actions selected by OFUL are solutions to the following constrained optimization problem
\begin{align*}
  &\bx_t = \argmax_{\bx \in D_t} \max_{\bw \in \reals^d} \bx\tp \bw\\
  &\text{such that} \quad
\|\bw - \bhatw_{t-1}\|_{\bV_{t-1}} \leq \beta_{t-1}(\delta)~.
\end{align*}
Using Lemma~\ref{lem:simple}, OFUL can be formulated as Algorithm~\ref{alg:oful}.
\begin{algorithm}
\caption{(OFUL)}
\label{alg:oful}
\begin{algorithmic}[1]
\Require{$\delta,\lambda > 0$}
\State $\bhatw_0 = \bzero, \bV_0^{-1} = \frac{1}{\lambda} \bI$.
\For{$t = 1,2,\ldots$}
\State Get decision set $D_t$
\State Play ${\dt \bx_t \gets \argmax_{\bx\in D_t} \cbr{ \bhatw_{t-1}^{\top}\bx + \beta_{t-1}(\delta) \norm{\bx}_{\bV_{t-1}^{-1}} } }$
\State Observe reward $Y_t$
\State Compute $\bV_t^{-1}$ and $\bhatw_t$ using~\eqref{eq:rls}
\EndFor
\end{algorithmic}
\end{algorithm}
Note that $\bx_t$ maximizes the expected reward estimate $\bhatw_{t-1}^{\top}\bx$ plus a term $\beta_{t-1}(\delta) \norm{\bx}_{\bV_{t-1}^{-1}}$ that provides an upper confidence bound for the RLS estimate in the direction of $\bx$.

\paragraph{Linear TS.} The linear Thompson Sampling algorithm of \cite{agrawal2013thompson} is Bayesian in nature: the selected actions and the observed rewards are used to update a Gaussian prior over the parameter space. Each action $\bx_t$ is selected by maximixing $\bx\tp\bhatw_t\ts$ over $\bx\in D_t$, where $\bhatw_t\ts$ is a random vector drawn from the posterior. As shown by \cite{abeille2017linear}, linear TS can be equivalently defined as a randomized algorithm based on the RLS estimate (see Algorithm~\ref{alg:ts}).
\begin{algorithm}
\caption{(Linear TS)}
\label{alg:ts}
\begin{algorithmic}[1]
\Require{$\delta,\lambda > 0, m \in \{1,\ldots, d-1\}$, $\sD\ts$ (sampling distribution)}
\State $\bhatw_0 = \bzero, \bV_0^{-1} = \frac{1}{\lambda} \bI_{d \times d}, \delta' = \delta / (4 T)$
\For{$t = 1,2,\ldots$}
\State Get decision set $D_t$
\State Sample $\boldeta_t \sim \sD\ts$
\State Play ${\dt \bx_t \gets \argmax_{\bx \in D_t} \bx\tp \Big(\bhatw_{t-1} + \tilde{\beta}_t(\delta') \bV_{t-1}^{-\frac{1}{2}} \boldeta_t\Big) }$
\State Observe reward $Y_t$
\State Compute $\bV_t^{-\frac{1}{2}}$ and $\bhatw_t$ using~\eqref{eq:rls}
\EndFor
\end{algorithmic}
\end{algorithm}
The random vectors $\boldeta_t$ are drawn i.i.d.\ from a suitable multivariate distribution $\sD\ts$ that need not be related to the posterior. In order to prove regret bounds, it is sufficient that the law of $\boldeta_t$ satisfies certain properties.
\begin{definition}[\acs{TS}-sampling distribution]
\label{def:ts_dist}
  A multivariate distribution $\sD\ts$ on $\reals^d$, absolutely continuous w.r.t.\ the Lebesgue measure, is \acs{TS}-sampling if it satisfies the following two properties:
  \begin{itemize}[topsep=0pt,parsep=0pt,itemsep=0pt]
  \item (Anti-concentration) There exists $p > 0$ such that for any $\bu$ with $\norm{\bu}=1$,
    $
    \P\big(\bu\tp \boldeta \geq 1\big) \geq p
    $.
  \item (Concentration) There exist $c,c' > 0$ such that for all $\delta \in (0,1)$,
    \[
    \P\pr{\|\boldeta\| \leq \sqrt{c d \ln\pr{\frac{c' d}{\delta}}}} \geq 1 - \delta~.
    \]
  \end{itemize}
\end{definition}
Similarly to OFUL, linear \acs{TS} uses the notion of confidence ellipsoid. However, due to the properties of the sampling distribution $\sD\ts$, the ellipsoid used by linear \acs{TS} is larger by a factor of order $\sqrt{d}$ than the ellipsoid used by OFUL. This causes an extra factor of $\sqrt{d}$ in the regret bound, which is not known to be necessary.

Note that both OFUL and linear TS need to maintain $\bV_t^{-1}$ (or $\bV_t^{-\frac{1}{2}}$), which requires time $\Omega\big(d^2\big)$ to update. In the next section, we show how this update time can be improved by sketching the regularized correlation matrix $\bV_t$.

\section{Sketching the correlation matrix}
\label{sec:skoful}
The idea of sketching is to maintain an approximation of $\bX_t$, denoted by $\bS_t \in \reals^{m \times d}$, where $m \ll d$ is a small constant called the sketch size. If we choose $m$ such that $\bS_t\tp \bS_t$ approximates $\bX_t\tp \bX_t$ well, we could use $\bS_t\tp \bS_t + \lambda \bI$ in place of $\bV_t$.
In the following we use the notation
$
  \btilV_t = \bS_t\tp \bS_t + \lambda \bI
$
to denote the sketched regularized correlation matrix. The \acs{RLS} estimate based upon it is denoted by
\begin{equation}
  \label{eq:btilw}
  \btilw_t = \btilV_t^{-1} \sum_{s=1}^t \bx_s Y_s~.
\end{equation}
A trivial replacement of $\bV$ with $\btilV$ does not yield an efficient algorithm. On the other hand, using the Woodbury identity we may write
\begin{align*}
  \btilV_t^{-1} = \frac{1}{\lambda} \pr{\bI_{d \times d} - \bS_t\tp \bH_t \bS_t}
\end{align*}
where $\bH_t = \pr{\bS_t \bS_t\tp + \lambda \bI_{m \times m}}^{-1}$.
Here matrix-vector multiplications involving $\bS_t$ require time $\scO(m d)$, while matrix-matrix multiplications involving $\bH_t$ require time $\scO(m^2)$. So, as long as $\bS_t$ and $\bH_t$ can be efficiently maintained, we obtain an algorithm for linear stochastic bandits where $\btilV_t^{-1}$ can be updated in time $\scO(m d + m^2)$. Next, we focus on a concrete sketching algorithm that ensures efficient updates of $\bS_t$ and $\bH_t$.
\paragraph{Frequent Directions.}
\ac{FD} \citep{ghashami2016frequent} is a deterministic sketching algorithm that maintains a matrix $\bS_t$ whose last row is invariably $\bzero$. On each round, we insert $\bx_t\tp$ into the last row of $\bS_{t-1}$, perform an eigendecomposition
$
  \bS_{t-1}\tp \bS_{t-1} + \bx_t \bx_t\tp = \bU_t\,\bSigma_t\,\bU_t\tp
$,
and then set $\bS_t = \big(\bSigma_t - \rho_t \bI_{m \times m}\big)^{\frac{1}{2}} \bU_t$, where $\rho_t$ is the smallest eigenvalue of $\bS_t\tp \bS_t$.
Observe that the rows of $\bS_t$ form an orthogonal basis, and therefore $\bH_t$ is a diagonal matrix which can be updated and stored efficiently. Now, the only step in question is an eigendecomposition, which can also be done in time $\scO(m d)$ ---see~\citep[Section 3.2]{ghashami2016frequent}. Hence, the total update time per round is $\scO(m d)$. The updates of matrices $\bS_t$ and $\bH_t$ are summarized in Algorithm~\ref{alg:fd_sketching}.
\begin{algorithm}[H]
\caption{(\ac{FD} Sketching)}
\label{alg:fd_sketching}
\begin{algorithmic}[1]
\Require{$\bS_{t-1} \in \reals^{m \times d}, \bx_t \in \reals^d, \lambda > 0$}
\State \text{Compute eigendecomposition} $\bU\tp \diag{\rho_1,\ldots,\rho_m} \bU = \bS_{t-1}\tp \bS_{t-1} + \bx_t \bx_t\tp$
\State $\bS_t \gets \diag{\sqrt{\rho_1-\rho_m},\ldots,\sqrt{\rho_{m-1}-\rho_m},0} \bU$
\State $\bH_t \gets \diag{ \frac{1}{\rho_1 - \rho_m + \lambda}, \ldots, \frac{1}{\lambda} }$
\Ensure{$\bS_t, \bH_t$}
\end{algorithmic}
\end{algorithm}
It is not hard to see that \ac{FD} sketching sequentially identifies the top-$m$ eigenvectors of the matrix $\bX_T\tp \bX_T$. Thus, whenever we use a sketched estimate, we lose a part of the spectrum tail. This loss is captured by the following notion of \emph{spectral error},
\begin{equation}
\label{eq:eps}
\ve_m = \min_{k=0,\dots,m-1}\frac{\lambda_{d-k} + \lambda_{d-k+1} + \cdots + \lambda_d}{\lambda (m-k)}
\end{equation}
where $\lambda_1 \geq \ldots \geq \lambda_d$ are the eigenvalues of the correlation matrix $\bX_T\tp \bX_T$. Note that $\ve_m \le (\lambda_m+\cdots+\lambda_d)/\lambda$. For matrices with low rank or light-tailed spectra we expect this spectral error to be small. In the following, we use $\mtil$ to denote the quantity $m + d \ln(1 + \ve_m)$ which occurs often in our bounds involving sketching. Note that $\mtil \ge m$ and $\mtil\to m$ as the spectral error vanishes.

Since the matrix $\bV_t$ is used to compute both the RLS estimate $\bhatw_t$ and the norm $\norm{\cdot}_{\bV_t}$, the sketching of $\bV_t$ clearly affects the confidence ellipsoid. The next theorem quantifies how much the confidence ellipsoid must be blown up in order to compensate for the sketching error. Let $\rho_t$ be the smallest eigenvalue of the FD-sketched correlation matrix $\bS_t\tp \bS_t$ and let $\bar{\rho}_t = \rho_1 + \cdots + \rho_t$. The following proposition due to \cite{ghashami2016frequent} (see the proof of Thm.~3.1, bound on $\Delta$) relates $\bar{\rho}_t$ to $\ve_m$ defined in~\eqref{eq:eps}.
\begin{proposition}
\label{prop:ve}
For any $t=0, \ldots, T$, any $\lambda > 0$, and any sketch size $m = 1, \ldots, d$, it holds that $\bar{\rho}_t/\lambda \le \ve_m$.
\end{proposition}
\vspace{-2mm}
A key lemma in the analysis of regret is the following sketched version of \cite[Lemma~11]{abbasi2011improved}, which bounds the sum of the ridge leverage scores. Although sketching introduces the spectral error $\ve_m$, it also improves the dependence on the dimension from $d$ to $m$ whenever $\ve_m$ is sufficiently small.
\begin{lemma}[Sketched leverage scores]
\label{lem:compressed:potential}
\begin{align}
    &\sum_{t=1}^T \min\cbr{1, \|\bx_t\|_{\btilV_{t-1}^{-1}}^2}
    \leq 2 \pr{1 + \ve_m} \pr{ \mtil + m \ln\pr{1 + \frac{T L^2}{m \lambda}} }~. \label{eq:compressed:potential-1}
\end{align}
\end{lemma}
We can now state the main result of this section.
\begin{theorem}[Sketched confidence ellipsoid]
\label{thm:confidence}
Let $\btilw_t$ be the RLS estimate constructed by an arbitrary policy for linear contextual bandits after $t$ rounds of play. For any $\delta \in (0, 1)$, the optimal parameter $\bwstar$ belongs to the set
$
\widetilde{C}_t \equiv \cbr{ \bw \in \R^d ~:~ \|\bw - \btilw_t\|_{\btilV_t} \leq \widetilde{\beta}_t(\delta) }
$
with probability at least $1-\delta$, where
\begin{align}
\widetilde{\beta}_t(\delta) &= R \sqrt{m \ln\pr{1 + \frac{t L^2}{m \lambda}} + 2 \ln\frac{1}{\delta} + d \ln\pr{1 + \frac{\bar{\rho}_t}{\lambda}}}
                            \cdot \sqrt{1 +  \frac{\bar{\rho}_t}{\lambda} } + S \sqrt{\lambda} \patch{ \pr{1 + \frac{1}{\lambda}}} \pr{1 + \frac{\bar{\rho}_t}{\lambda}} \label{eq:beta_tilde}\\
                            &\defOtilde R \sqrt{\mtil \pr{1 + \ve_m}} + S \sqrt{\lambda} \patch{ \pr{1 + \frac{1}{\lambda}}} \pr{1 + \ve_m}~. \label{eq:beta_tilde_bound}
\end{align}
\end{theorem}
Note that~\eqref{eq:beta_tilde_bound} is larger than its non-sketched counterpart~\eqref{eq:beta} due to the factors $1 + \ve_m$. However, when the spectral error $\ve_m$ vanishes, $\tilde{\beta}_t(\delta)$ becomes of order $R \sqrt{m} + S \sqrt{\lambda} \patch{ \pr{1 + \frac{1}{\lambda}}}$, which improves upon~\eqref{eq:beta} since we replace the dependence on the ambient space dimension $d$ with the dependence on the sketch size $m$.
In the following, we use the abbreviation $M_{\lambda} = \max\big\{1, 1/\sqrt{\lambda}\big\}$.
\section{Sketched OFUL}
Equipped with the sketched confidence ellipsoid and the sketched RLS estimate, we can now introduce SOFUL (Algorithm~\ref{alg:skoful}), the sketched version of OFUL.
\begin{algorithm}
\caption{(SOFUL)}
\label{alg:skoful}
\begin{algorithmic}[1]
\Require{$\delta,\lambda > 0, m \in \{1,\ldots,d-1\}$}
\State $\btilw_0 = \bzero, \btilV_0^{-1} = \frac{1}{\lambda} \bI_{d \times d}, \bS_0 = \bzero_{m \times d}$
\For{$t = 1,2,\ldots$}
\State Get decision set $D_t$
\State Play ${\dt \bx_t \gets \argmax_{\bx\in D_t} \cbr{ \btilw_{t-1}\tp \bx + \tilde{\beta}_{t-1}(\delta) \norm{\bx}_{\btilV_{t-1}^{-1}} } }$
\State Observe reward $Y_t$
\State Compute $\bS_t, \bH_t$ using Alg.~\ref{alg:fd_sketching} given $\bS_{t-1}, \bx_t$
\State $\btilV_t^{-1} \gets \frac{1}{\lambda} \pr{\bI_{d \times d} - \bS_t\tp \bH_t \bS_t}$
\State Compute $\btilw_t$ using~\eqref{eq:btilw}
\EndFor
\end{algorithmic}
\end{algorithm}
SOFUL enjoys the following regret bound, characterized in terms of the spectral error.
\begin{theorem}
\label{thm:regret}
The regret of SOFUL with \ac{FD}-sketching of size $m$ w.h.p.\ satisfies
\[
  R_T \defOtilde M_{\lambda} \big(1 + \ve_m\big)^{\frac{3}{2}} \mtil \pr{R + S \sqrt{\lambda} \patch{ \pr{1 + \frac{1}{\lambda}}}} \sqrt{T}~.
\]
\end{theorem}
Similarly to \cite{abbasi2011improved}, we also prove a distribution dependent regret bound for SOFUL. This bound is polylogarithmic in time and depends on the smallest difference $\Delta$ between the rewards of the best and the second best action in the decision sets,
\[
  \Delta = \min_{t=1,\ldots,T} \max_{\bx\in D_t\setminus\{\bxstar_t\}} \pr{\bxstar_t - \bx}\tp \bwstar~.
\]
\begin{theorem}
\label{thm:regret_star}
The regret of SOFUL with FD-sketching of size $m$ w.h.p.\ satisfies
\[
  R_T \defOtilde M_{\lambda} \pr{1 + \ve_m}^3 \mtil^2  \pr{R^2 + S^2 \lambda \patch{\pr{1 + \frac{1}{\lambda}}^2}} \frac{(\ln T)^2}{\Delta}~.
\]
\end{theorem}
Proofs of the regret bounds appear in the supplementary material (Section~\ref{sec:regret_proofs}).

\section{Sketched linear TS}
\label{sec:thompson}
In this section we introduce a variant of linear \acs{TS} (Algorithm~\ref{alg:ts}) based on FD-sketching. Similarly to SOFUL, sketched linear \acs{TS} (see Algorithm~\ref{alg:sketched_ts}) uses the FD-sketched approximation $\btilV_{t-1}$ of the correlation matrix $\bV_{t-1}$ in order to select the action $\bx_t$.
\begin{algorithm}
\caption{(Sketched linear \acs{TS})}
\label{alg:sketched_ts}
\begin{algorithmic}[1]
\Require{$\delta,\lambda > 0, m \in \{1,\ldots, d-1\}$, $\sD\ts$ (\acs{TS}-sampling distribution)}
\State $\btilw_0 = \bzero, \btilV_0^{-1} = \frac{1}{\lambda} \bI_{d \times d}, \bS_0 = \bzero_{m \times d}, \delta' = \delta / (4 T)$
\For{$t = 1,2,\ldots$}
\State Get decision set $D_t$
\State Sample $\boldeta_t \sim \sD\ts$
\State Play ${\dt \bx_t \gets \argmax_{\bx \in D_t} \bx\tp \Big(\btilw_{t-1} + \tilde{\beta}_t(\delta') \btilV_{t-1}^{-\frac{1}{2}} \boldeta_t\Big) }$
\State Observe reward $Y_t$
\State Compute $\bS_t, \bH_t$ using Algorithm~\ref{alg:fd_sketching} given $\bS_{t-1}, X_t$
\State $\btilV_t^{-1} \gets \frac{1}{\lambda} \pr{\bI_{d \times d} - \bS_t\tp \bH_t \bS_t}$
\State Compute $\btilw_t$ using~\eqref{eq:btilw}
\EndFor
\end{algorithmic}
\end{algorithm}
Note that, in this case, we need both $\btilV_{t-1}^{-1}$ and $\btilV_{t-1}^{-\frac{1}{2}}$ to compute $\bx_t$. Using the generalized Woodbury identity (Corollary~\ref{cor:wood-use} in Section~\ref{sec:sketch_tools} for proofs), we can write
\[
	\btilV_t^{-\frac{1}{2}}
=
	\bS_t^{'\top} \pr{\bS_t' \bS_t^{'\top}}^{-1} \pr{\frac{\lambda}{2} \bI + \bS_t' \bS_t^{'\top}}^{-\frac{1}{2}} \bS_t'
\]
where
\[
    \bS_t' = \pr{\bSigma_t + \pr{\frac{\lambda}{2} - \rho_t} \bI_{m\times m}}^{\frac{1}{2}} \bU_t~.
\]
Note that $\btilV_t^{-\frac{1}{2}}$ can still be computed in time $\scO\big(md + m^2\big)$ because $\bS_t' \bS_t^{'\top}$ is a diagonal matrix.

The confidence ellipsoid stated in Theorem~\ref{thm:confidence} applies to any contextual bandit policy, and so also to the $\btilw_t$ constructed by sketched linear TS. However, as shown by \cite{abeille2017linear}, the analysis needs a confidence ellipsoid larger by a factor equal to the bound on $\norm{\boldeta}$ appearing in the concentration property of the TS-sampling distribution. More precisely, the \textsl{TS-confidence ellipsoid} is defined by 
\[
    \tilde{C}_t\ts \equiv \cbr{\bw \in \reals^d ~:~ \|\bw - \btilw_t\|_{\btilV_t} \leq \tilde{\gamma}_t\big(\delta/(4T)\big)}
\]
where
\begin{align}
\label{eq:gamma_tilde_bound}
  \tilde{\gamma}_t(\delta) = \tilde{\beta}_t(\delta) \sqrt{c d \ln\pr{\frac{c' d}{\delta}}}~.
\end{align}
The quantity $\tilde{\beta}_t(\delta)$ is defined in~\eqref{eq:beta_tilde} and $c,c'$ are the concentration constants of the TS-sampling distribution (Definition~\ref{def:ts_dist}). We are now ready to prove a bound on the regret of linear TS with FD-sketching.
\begin{theorem}
\label{thm:regret_ts}
The regret of FD-sketched linear TS, run with sketch size $m$ w.h.p.\ satisfies
\[
  R_T \defOtilde M_{\lambda} \pr{1 + \ve_m}^{\frac{3}{2}} \mtil \pr{ R + S \sqrt{\lambda} \patch{ \pr{1 + \frac{1}{\lambda}} } } \sqrt{dT}~.
\]
\end{theorem}
The proof of Theorem~\ref{thm:regret_ts} closely follows the analysis of~\cite{abeille2017linear} with some key modifications due to the sketching operations. For completeness, we include the proof in Section~\ref{sec:thompson_proof}.

\section{Proofs}
\label{sec:proofs}
We start with the proof of a simple lemma that is used in the definition of OFUL (see Algorithm~\ref{alg:oful}).
\begin{lemma}
\label{lem:simple}
For any positive definite $d\times d$ matrix $\bA$, for any $\bw_0, \bx\in\R^d$ and $c > 0$, the solution of
\begin{align*}
	\max_{\bw\in\R^d} & \quad\bw^{\top}\bx
\\
	\text{s.t.} & \quad \norm{\bw-\bw_0}_{\bA} \le c
\end{align*}
has value 
$
	\bw_0^{\top}\bx + c\norm{\bx}_{\bA^{-1}}
$.
\end{lemma}
\begin{proof}
Let $\bu = \bA^{\frac{1}{2}}(\bw-\bw_0)$ so that $\bw = \bA^{-\frac{1}{2}}\bu + \bw_0$. Then the optimization problem can be equivalently rewritten as
\begin{align*}
	\max_{\bw\in\R^d} & \quad\bu^{\top}\bA^{-\frac{1}{2}}\bx + \bw_0^{\top}\bx
\\
	\text{s.t.} & \quad \norm{\bu} \le c
\end{align*}
Then the solution is clearly $\bu = c\,\bA^{-\frac{1}{2}}\bx\big/\norm{\bx}_{\bA^{-1}}$, which achieves the claimed value.
\end{proof}
Our regret analyses follow \citep{abbasi2011improved,abeille2017linear} and related works. However, due to the sketching of the correlation matrix, some key components of the proofs now depend on the spectral error~\eqref{eq:eps}. In Section~\ref{sec:sketch_tools}, we present tools specific to the analysis of linear bandits with \ac{FD}-sketching. These tools are used to bound the instantaneous regret $\big(\bxstar - \bx_t\big)\tp \bwstar$ in terms of the norm $\|\bwstar - \btilw_t\|_{\btilV_{t-1}}$ and the ridge leverage scores $\sum_{t=1}^T \|\bx_t\|_{\btilV_{t-1}^{-1}}^2$. Armed with these results, we then prove our regret bounds in Sections~\ref{sec:regret_proofs} and~\ref{sec:thompson_proof}.

Next, we recall some standard tools from the analysis of linear bandits. All results in Section~\ref{sec:std_tools} are by~\cite{abbasi2011improved}.
\subsection{Tools from the analysis of linear contextual bandits}
Recall that $\bV_t = \sum_{s=1}^t \bx_s \bx_s\tp + \lambda \bI$ with $\lambda > 0$.
\label{sec:std_tools}
\begin{lemma}[Determinant-trace inequality]
\label{lem:det_trace}
\[
	\ln \det\pr{\bV_t} \leq d \ln \pr{\lambda + \frac{t L^2}{d}}~.
\]
\end{lemma}
\begin{lemma}[Ridge leverage scores]
\label{lem:potential_d}
\begin{equation}
\label{eq:elliptic_potential_1}
\sum_{t=1}^T \min\cbr{1, \|\bx_t\|_{\bV_{t-1}^{-1}}^2} \leq 2 \ln\pr{ \frac{\det\pr{\bV_T}}{\lambda \bI} }~.
\end{equation}
For $\lambda \geq \max\cbr{1, L^2}$, we also have that
\begin{equation}
\label{eq:elliptic_potential_2}
\sum_{t=1}^T \|\bx_t\|_{\bV_{t-1}^{-1}}^2 \leq 2 d \ln\pr{1 + \frac{T L^2}{\lambda d} }~.
\end{equation}
\end{lemma}
\begin{theorem}[Self-normalized bound for vector-valued martingales]
\label{thm:martingales}
Let
\[
	S_t = \sum_{s=1}^t \eta_s \bx_s \qquad t \ge 1
\]
where $\eta_1,\eta_2,\ldots$ is a conditionally $R$-subgaussian real-valued stochastic process and $\bx_1,\bx_2,\dots$ is any $\reals^d$-valued stochastic process such that $\bx_t$ is measurable with respect to the $\sigma$-algebra generated by $\eta_1,\dots,\eta_{t-1}$. Then, for any $\delta > 0$, with probability at least $1-\delta$,
$
	\norm{S_t}_{\bV_t^{-1}}^2 \leq B_t(\delta)
$
for all $t \ge 0$, where
\begin{equation}
\label{eq:B_t}
	B_t(\delta) = 2 R^2 \ln\pr{\frac{1}{\delta} \det\pr{\bV_t}^{\frac{1}{2}} \det\pr{\lambda \bI}^{-\frac{1}{2}} }~.
\end{equation}
\end{theorem}
Theorem~\ref{thm:martingales} is key to showing that $\bwstar$ lies within the confidence ellipsoid centered at the estimate $\btilw_t$ at time step $t$, this irrespective of the process that selected the $\bx_s$ used to build $\btilw_t$.

\subsection{Linear algebra and sketching tools}
\label{sec:sketch_tools}
We start by introducting a basic relationship between the correlation matrix of actions $\bX_s\tp \bX_s$ and its \ac{FD}-sketched estimate $\bS_t\tp \bS_t$ with sketch size $m \le d$. Recall that $\rho_t$ is the smallest eigenvalue of $\bS_t\tp \bS_t$ for $t=1,\ldots,T$ and $\bar{\rho}_t = \rho_1 + \cdots + \rho_t$. Recall also that $\btilV = \bS_t\tp \bS_t + \lambda \bI$.
\patch{ 
\begin{proposition}
\label{prop:fd_XX_SS_rho}
Let $\bS_s$ be the matrix computed by FD-sketching at time step $s=1,\ldots,t$ (where $\bS_0 = \bzero$). Then
\begin{align*}
    \bX_s\tp\bX_s = \bS_s\tp \bS_s + \sum_{k=1}^s \rho_k \bU_k \bU_k\tp
\end{align*}
where $\bU_k \in \reals^{d \times m}$ is a matrix of eigenvectors of $\bS_{k-1}\tp \bS_{k-1} + \bx_k \bx_k\tp$.
Moreover,
\begin{align*}
    \bS_s\tp \bS_s \preceq \bX_s\tp\bX_s \preceq \bS_s\tp \bS_s + \bar{\rho}_s \bI
  \end{align*}
\end{proposition}
\begin{proof}
By construction,
$
\bS_{s-1}\tp \bS_{s-1} + \bx_s \bx_s\tp = \bU_s \bSigma_s \bU_s\tp
$
where $\bS_s = \pr{\bSigma_s - \rho_s \bI_{m \times m}}^{\frac{1}{2}} \bU_s$.
Thus,
\begin{align*}
\bS_s\tp \bS_s = \bU_s \bSigma_s \bU_s\tp - \rho_s \bI = \bS_{s-1}\tp \bS_{s-1} + \bx_s \bx_s\tp - \rho_s \bU_s \bU_s\tp~.
\end{align*}
Summing both sides of the above over $s=1,\ldots,t$ we get
\begin{align*}
\bS_t\tp \bS_t = \sum_{s=1}^t \bx_s \bx_s\tp - \sum_{s=1}^t \rho_s \bU_s \bU_s\tp
\end{align*}
which implies the desired result.
\end{proof}
} 
In the following lemma, we show a sketch-specific version of the determinant-trace inequality (Lemma~\ref{lem:det_trace}). When the spectral error is small, the right-hand side of the inequality depends on the sketch size $m$ rather than the ambient dimension $d$.
\begin{lemma}
\label{lem:det_trace_m}
\[
	\ln\pr{\frac{\det(\bV_t)}{\det(\lambda \bI)}}
\le
	d \ln\pr{1 + \frac{\bar{\rho}}{\lambda}} + m \ln\pr{1 + \frac{t L^2}{m \lambda}}~.
\]
\end{lemma}
\begin{proof}
  Let $\tilde{\lambda}_1, \tilde{\lambda}_2, \ldots, \tilde{\lambda}_d \geq 0$ be the eigenvalues of $\bS_t\tp \bS_t$. We start by looking at the ratio of determinants.
  Throughout the proof, unless stated explicitly, denote $\bar{\rho} = \bar{\rho}_{t-1}$. 
  Using Proposition~\ref{prop:fd_XX_SS_rho} we can write
\patch{ 
  \begin{align}
  \frac{\det(\bV_t)}{\det(\lambda \bI)}
    &=
  \frac{\det\!\big(\bS_t\tp \bS_t + \sum_{s=1}^t \rho_s \bU_s \bU_s\tp + \lambda \bI\big)}{\det(\lambda \bI)} \nonumber\\
&\leq
                                                                                                                            \frac{\det\big(\bS_t\tp \bS_t + \bar{\rho} \bI + \lambda \bI\big)}{\det(\lambda \bI)} \nonumber\\
&=
\prod_{i=1}^d \pr{ \frac{\tilde{\lambda}_i}{\lambda} + 1 + \frac{\bar{\rho}}{\lambda} } \nonumber\\
                                             &= \pr{ 1 + \frac{\bar{\rho}}{\lambda} }^{d-m} \prod_{i=1}^m \pr{ \frac{\tilde{\lambda}_i}{\lambda} + 1 + \frac{\bar{\rho}}{\lambda} } \label{eq:det_trace_sketch_1}
  \end{align}
  }
since $\tilde{\lambda}_{m+1} = \cdots = \tilde{\lambda}_d = 0$ because $\bS_t\tp \bS_t$ has rank at most $m$. We now use the AM-GM inequality, stating that
\[
\pr{\prod_{i=1}^m \alpha_i}^{\frac{1}{m}} \leq \frac{1}{m} \sum_{i=1}^m \alpha_i \qquad \forall\; \alpha_1, \ldots, \alpha_m \ge 0~.
\]
Using the AM-GM inequality, the product in~\eqref{eq:det_trace_sketch_1} can be bounded as
\begin{align}
  \prod_{i=1}^m \pr{ \frac{\tilde{\lambda}_i}{\lambda} + 1 + \frac{\bar{\rho}}{\lambda} }
  &\leq \pr{1 + \frac{\bar{\rho}}{\lambda} + \frac{1}{m \lambda} \sum_{i=1}^m \tilde{\lambda}_i}^m \nonumber\\
  &= \pr{1 + \frac{\bar{\rho}}{\lambda} + \frac{\tr(\bS_t\tp \bS_t)}{m \lambda}}^m \nonumber\\
  &\leq \pr{1 + \frac{\bar{\rho}}{\lambda} + \frac{t L^2}{m \lambda}}^m \label{eq:det_trace_sketch_2}
\end{align}
where the last inequality holds because

\begin{align*}
  \tr(\bS_t\tp \bS_t) &= \tr\pr{\btilV_t - \lambda \bI}\\
                &\leq \tr\pr{\bV_t - \lambda \bI} \tag{by Proposition~\ref{prop:fd_XX_SS_rho}}\\
                &= \sum_{s=1}^t \tr(\bx_s \bx_s\tp) \tag{by definition of $\bV_t$}\\
                &\leq t L^2~.
\end{align*}

Finally, substituting~\eqref{eq:det_trace_sketch_2} into~\eqref{eq:det_trace_sketch_1} and taking logs on both sides gives
\begin{align*}
  \ln\pr{\frac{\det(\bV_t)}{\det(\lambda \bI)}} &\leq
(d-m) \ln\pr{1 + \frac{\bar{\rho}}{\lambda}} +
    m \ln\pr{1 + \frac{\bar{\rho}}{\lambda} + \frac{t L^2}{m \lambda}}\\
&= d \ln\pr{1 + \frac{\bar{\rho}}{\lambda}} +
    m \ln\pr{1 + \frac{\frac{t L^2}{m \lambda}}{1 + \frac{\bar{\rho}}{\lambda}}}\\
&\leq d \ln\pr{1 + \frac{\bar{\rho}}{\lambda}} +
    m \ln\pr{1 + \frac{t L^2}{m \lambda}}
\end{align*}
concluding the proof.
\end{proof}
The next lemma is similar to \cite[Lemma~11]{abbasi2011improved}. However, now the statement depends on the sketched matrix $\btilV_{t-1}$ instead of $\bV_{t-1}$.
Although we pay in terms of the spectral error $\ve_m$, we also improve the dependence on the dimension from $d$ to $m$ whenever $\ve_m$ is sufficiently small.
\begin{lemma}[Sketched leverage scores]
\label{lem:potential}
  \begin{align}
    \sum_{t=1}^T \min\cbr{1, \|\bx_t\|_{\btilV_{t-1}^{-1}}^2} \label{eq:potential-1}
    \leq 2 \pr{1 + \ve_m} \pr{ d \ln\pr{1 + \ve_m} + m \ln\pr{1 + \frac{T L^2}{m \lambda}} }~.
\end{align}
\end{lemma}
\begin{proof}
  Throughout the proof, unless stated explicitly, we drop the subscripts containing $t$. Therefore, $\bV = \bV_{t-1}$, $\btilV = \btilV_{t-1}$, $\bx = \bx_t$, and $\bar{\rho} = \bar{\rho}_{t-1}$. 
  Consider a following decomposition:
\patch{ 
  \begin{align*}
    \|\bx\|_{\bV^{-1}}^2
    &= \bx\tp \pr{\btilV + \sum_{s=1}^t \rho_s \bU_s \bU_s\tp}^{-1} \bx \tag{by Proposition \ref{prop:fd_XX_SS_rho}}\\
    &\geq \bx\tp \pr{\btilV + \bar{\rho} \bI }^{-1} \bx \tag{since $\btilV$ is \ac{PSD}}\\
    &=\bx\tp \btilV \btilV^{-1} (\btilV + \bar{\rho} \bI )^{-1} \bx\\
    &= \bx\tp \pr{ \sum_{i=1}^d \btilu_i \btilu_i\tp
      \, \frac{1}{\tilde{\lambda}_i + \lambda}
      \, \frac{\tilde{\lambda}_i + \lambda}{\tilde{\lambda}_i + \lambda + \bar{\rho}} } \bx\\
    &\geq \frac{\lambda}{\lambda + \bar{\rho}} \bx\tp \pr{ \sum_{i=1}^d \btilu_i \btilu_i\tp
      \frac{1}{\tilde{\lambda}_i + \lambda} } \bx\\
    &= \frac{\lambda}{\lambda + \bar{\rho}} \|\bx\|_{\btilV^{-1}}^2~.
  \end{align*}
  }
Furthermore, this implies that
\begin{align*}
  &\min \cbr{1, \frac{\lambda}{\lambda + \bar{\rho}} \|\bx\|_{\btilV^{-1}}^2} \leq \min\cbr{1, \|\bx\|_{\bV^{-1}}^2}\\
  \Longrightarrow \quad &\min \cbr{1 + \frac{\bar{\rho}}{\lambda}, \|\bx\|_{\btilV^{-1}}^2} \leq \pr{1 + \frac{\bar{\rho}}{\lambda}} \min\cbr{1, \|\bx\|_{\bV^{-1}}^2} \tag{multiply both sides by $1 + \frac{\bar{\rho}}{\lambda}$}\\
  \Longrightarrow \quad &\min \cbr{1, \|\bx\|_{\btilV^{-1}}^2} \leq \pr{1 + \frac{\bar{\rho}}{\lambda}} \min\cbr{1, \|\bx\|_{\bV^{-1}}^2}~.
\end{align*}
Finally, combining the above with Lemma~\ref{lem:potential_d}, equation~\eqref{eq:elliptic_potential_2}, and using the fact that $\bar{\rho}_{t-1} \leq \bar{\rho}_T$, we obtain
\begin{align*}
\sum_{t=1}^T \min\cbr{1, \|\bx_t\|_{\btilV_{t-1}^{-1}}^2}
&\leq 2 \pr{1 + \frac{\bar{\rho}_T}{\lambda}} \ln\pr{ \frac{\det(\bV_T) }{\det(\lambda I)} }\\
&\leq 2 \pr{1 + \frac{\bar{\rho}_T}{\lambda}} \pr{ d \ln\pr{1 + \frac{\bar{\rho}_T}{\lambda}} + m \ln\pr{1 + \frac{T L^2}{m \lambda}} } \tag{by Lemma~\ref{lem:det_trace_m}}\\
&\leq 2 \pr{1 + \ve_m} \pr{ d \ln\pr{1 + \ve_m} + m \ln\pr{1 + \frac{T L^2}{m \lambda}} }
\end{align*}
where the last inequality follows from Proposition~\ref{prop:ve}.
\end{proof}
Now we prove Theorem~\ref{thm:confidence}, characterizing the confidence ellipsoid generated by the sketched estimate.
\begin{nameddef}[Theorem~\ref{thm:confidence}]\emph{(Sketched confidence ellipsoid -- restated)}.
For any $\delta \in (0, 1)$, the optimal parameter $\bwstar$ belongs to the set
\[
\widetilde{C}_t \equiv \cbr{ \bw \in \R^d ~:~ \|\bw - \btilw_t\|_{\btilV_t} \leq \widetilde{\beta}_t(\delta) }
\]%
with probability at least $1-\delta$, where
\begin{align*}
\widetilde{\beta}_t(\delta) &= R \sqrt{m \ln\pr{1 + \frac{t L^2}{m \lambda}} + 2 \ln\pr{\frac{1}{\delta}} + d \ln\pr{1 + \frac{\bar{\rho}_t}{\lambda}}} \sqrt{1 +  \frac{\bar{\rho}_t}{\lambda} } + S \sqrt{\lambda} \patch{ \pr{1 + \frac{1}{\lambda}}} \pr{1 + \frac{\bar{\rho}_t}{\lambda}} \\
&\defOtilde R \sqrt{\pr{m + d \ln(1 + \ve_m)} \pr{1 + \ve_m}} + S \sqrt{\lambda} \patch{ \pr{1 + \frac{1}{\lambda}}} \pr{1 + \ve_m}~.
\end{align*}
\end{nameddef}

\begin{proof}
  Denote $\bM = \sum_{s=1}^k \rho_s \bU_s \bU_s\tp$.
  Throughout the proof we frequently use Proposition~\ref{prop:fd_XX_SS_rho}, implying
$
\bX_t\tp \bX_t = \bS_t\tp \bS_t + \bar{\rho}_t \bI
$.
For brevity, in the following we drop subscripts containing $t$ in matrices. Let $\myboldeta_t = (\eta_1, \eta_2 \ldots, \eta_t)$, and by definition of the sketched estimate we have that
\patch{ 
\begin{align}
\nonumber
  \btilw_t &= \pr{\bS_t\tp \bS_t + \lambda \bI}^{-1} \bX_t\tp \pr{\bX_t \bwstar + \myboldeta_t}\\
\nonumber
           &= \pr{\bS_t\tp \bS_t + \lambda \bI}^{-1} \bX_t\tp \myboldeta_t + \pr{\bS_t\tp \bS_t + \lambda \bI}^{-1} \bX_t\tp \bX_t \bwstar\\
\nonumber
           &= \pr{\bS_t\tp \bS_t + \lambda \bI}^{-1} \bX_t\tp \myboldeta_t\\
\nonumber
           &+ \pr{\bS_t\tp \bS_t + \lambda \bI}^{-1} \pr{ \bX_t\tp \bX_t + (\lambda - \bM) \bI } \bwstar
             - (\lambda \bI - \bM) \pr{\bS_t\tp \bS_t + \lambda \bI}^{-1} \bwstar\\
           &= \pr{\bS_t\tp \bS_t + \lambda \bI}^{-1} \bX_t\tp \myboldeta_t
             + \bwstar
             - (\lambda \bI - \bM) \pr{\bS_t\tp \bS_t + \lambda \bI}^{-1} \bwstar \label{eq:use_prop_2_1}\\
           &= \btilV_t^{-1} \bX_t\tp \myboldeta_t
             + \bwstar
             - (\lambda \bI - \bM) \btilV_t^{-1} \bwstar~. \label{eq:conf_ellipsoid_1}
\end{align}
Then, by~\eqref{eq:conf_ellipsoid_1}, for any $\bx \in \reals^d$ we have that
\begin{align}
  \bx\tp \big(\btilw_t - \bwstar\big)
  &= \bx\tp {\btilV_t}^{-1} \bX_t\tp \myboldeta_t
    +
    \bx\tp (\bM - \lambda \bI) {\btilV_t}^{-1} \bwstar
    \label{eq:conf_ellip_sk:bx_dot}\\
\nonumber
  &\leq \big\|\bx\tp {\btilV_t}^{-1}\big\|_{\bV_t} \|\bX_t\tp \myboldeta_t\|_{\bV_t^{-1}} \tag{by Cauchy-Schwartz}
    +
    \bx\tp (\bM - \lambda \bI) {\btilV_t}^{-1} \bwstar~.
\end{align}
We now choose $\bx = \btilV_t (\btilw_t - \bwstar)$ and proceed by bounding terms in the above. By the choice of $\bx$, we have that $\bx\tp \big(\btilw_t - \bwstar\big) = \|\btilw_t - \bwstar\|_{\btilV_t}^2$,
$\big\|\bx\tp {\btilV_t}^{-1}\big\|_{\bV_t} = \|\btilw_t - \bwstar\|_{\bV_t}$
and
\begin{align*}
  \bx\tp (\bM - \lambda \bI) \btilV_t^{-1} \bwstar
  &=
    (\btilw_t - \bwstar)\tp \btilV_t\tp (\bM - \lambda \bI) \btilV_t^{-1} \bwstar\\
  &\leq
    |\bar{\rho}_t + \lambda| \|\btilw_t - \bwstar\|_2 \|\btilV_t\|_2 \|\btilV_t^{-1}\|_2 \|\bwstar\|_2\\
    &\leq
      |\bar{\rho}_t + \lambda| \|\btilw_t - \bwstar\|_2 \pr{1 + \frac{1}{\lambda}} S
\end{align*}
where we used the fact that by definition of $\bM$, $\|\bM - \lambda \bI\|_2 = \|\sum_{s=1}^t \rho_s \bU_s \bU_s\tp - \lambda \bI\|_2 \leq |\bar{\rho}_s + \lambda|$.
}

Finally, by Theorem~\ref{thm:martingales}, for any $\delta > 0$, with probability at least $1-\delta$,
\[
  \|\bX\tp \myboldeta_t\|_{\bV_t^{-1}} \leq \sqrt{B_t(\delta)} \qquad \forall t \geq 0~.
\]
The left-hand side of~\eqref{eq:conf_ellip_sk:bx_dot} can now upper bounded as
\begin{align}
  &\|\btilw_t - \bwstar\|_{\btilV_t}^2 \leq \sqrt{B_t(\delta)} \|\btilw_t - \bwstar\|_{\bV_t} + S (\lambda + \bar{\rho}_t) \patch{\pr{1 + \frac{1}{\lambda}}} \|\btilw_t - \bwstar\|_2 \nonumber\\
  \Longrightarrow \qquad
  &\|\btilw_t - \bwstar\|_{\btilV_t} \leq \sqrt{B_t(\delta)} \frac{\|\btilw_t - \bwstar\|_{\bV_t}}{\|\btilw_t - \bwstar\|_{\btilV_t}} + S (\lambda + \bar{\rho}_t) \patch{\pr{1 + \frac{1}{\lambda}}} \frac{\|\btilw_t - \bwstar\|_2}{\|\btilw_t - \bwstar\|_{\btilV_t}}~. \label{eq:conf_ellip_sk:norm_ratio}
\end{align}
Now we handle the ratios of norms in the right-hand side of~\eqref{eq:conf_ellip_sk:norm_ratio}. First,
by Proposition \ref{prop:fd_XX_SS_rho},
\begin{align*}
  \frac{\|\btilw_t - \bwstar\|_{\bV_t}}{\|\btilw_t - \bwstar\|_{\btilV_t}}
&\patch{\leq} \sqrt{\frac{\|\btilw_t - \bwstar\|_{\btilV_t}^2 + \bar{\rho}_t \|\btilw_t - \bwstar\|_2^2}{\|\btilw_t - \bwstar\|_{\btilV_t}^2}}\\
&= \sqrt{ 1 + \bar{\rho}_t \frac{\|\btilw_t - \bwstar\|_2^2}{\|\btilw_t - \bwstar\|_{\btilV_t}^2} }\\
&\leq \sqrt{1 + \frac{\bar{\rho}_t}{\lambda}}
\end{align*}
since $\|\btilw_t - \bwstar\|_{\btilV_t}^2 \geq \lambda \|\btilw_t - \bwstar\|_2^2$ and, using the same reasoning,
\[
\frac{\|\btilw_t - \bwstar\|_2}{\|\btilw_t - \bwstar\|_{\btilV_t}} \leq \frac{1}{\sqrt{\lambda}}~.
\]
Substituting these into~\eqref{eq:conf_ellip_sk:norm_ratio} gives
\begin{align*}
  \|\btilw_t - \bwstar\|_{\btilV_t} \leq \sqrt{B_t(\delta) \pr{1 +  \frac{\bar{\rho}_t}{\lambda}} } + S \sqrt{\lambda} \patch{\pr{1 + \frac{1}{\lambda}}} \pr{1 + \frac{\bar{\rho}_t}{\lambda}}~.
\end{align*}
Now we provide a deterministic bound on $B_t(\delta)$. Using Lemma~\ref{lem:det_trace_m} we have
\begin{align*}
  \sqrt{B_t(\delta)} &= R \sqrt{2 \ln\pr{\frac{1}{\delta} \det\pr{\bV_t}^{\frac{1}{2}} \det\pr{\lambda \bI}^{-\frac{1}{2}} } }\\
  &\leq R \sqrt{d \ln\pr{1 + \frac{\bar{\rho}_t}{\lambda}} + m \ln\pr{1 + \frac{t L^2}{m \lambda}} + 2 \ln\pr{\frac{1}{\delta}}}~.
\end{align*}
This proves the first statement~\eqref{eq:beta_tilde}.
Finally, \eqref{eq:beta_tilde_bound} follows by Proposition~\ref{prop:ve}, that is
$
1 + {\bar{\rho}_t}/{\lambda} \leq 1 + \ve_m
$.
\end{proof}

We close this section by computing a closed form for $\btilV_t^{-\frac{1}{2}}$, the square root of the inverse of the sketched correlation matrix. This is used by sketched linear TS for selecting actions. We make use of the following result ---see. e.g., \cite[Theorem~1.35]{higham2008functions}.
\begin{theorem}[Generalized Woodbury matrix identity]
  \label{thm:wood-gen}
  Let $\bA \in \mathbb{C}^{d \times m}$ and $\bB \in \mathbb{C}^{m \times d}$, with $d \geq m$, and assume that $\bB \bA$ is nonsingular.
  Let $f$ be defined on the spectrum of $\alpha \bI_{d \times d} + \bA \bB$, and if $d = m$ let $f$ be defined at $\alpha$.
  Then
  $
    f(\alpha \bI_{d \times d} + \bA \bB) = f(\alpha \bI_{d \times d}) + \bA (\bB \bA)^{-1} \pr{f(\alpha \bI_{m \times m} + \bB \bA) - f(\alpha \bI_{m \times m})} \bB
  $.
\end{theorem}
This is used to prove the following.
\begin{corollary}
\label{cor:wood-use}
For $\lambda > 0$, let
\[
    \label{eq:S_proxy}
    \bS_t' = \pr{\bSigma_t + \pr{\frac{\lambda}{2} - \rho_t} \bI_{m\times m}}^{\frac{1}{2}} \bU_t~.
\]
Then
\[
	\btilV_t^{-\frac{1}{2}}
=
	\bS_t^{'\top} \pr{\bS_t' \bS_t^{'\top}}^{-1} \pr{\frac{\lambda}{2} \bI + \bS_t' \bS_t^{'\top}}^{-\frac{1}{2}} \bS_t'~.
\]
\end{corollary}
\begin{proof}
We apply Theorem~\ref{thm:wood-gen} with $f(\btilV) = \btilV_t^{-\frac{1}{2}}$. However, since $\bS_t \bS_t\tp$ is singular by design, 
we apply the theorem with $\bB$ set the non-singular proxy matrix $\bS_t'$, $\bA$ set to $\bS_t^{'\top}$, and $\alpha$ set to $\lambda/2$.
  Thus
  $
    \btilV_t = \bS_t^{'\top} \bS_t' + \frac{\lambda}{2} \bI_{d\times d}
  $
  and
  \begin{align}
    \pr{\bS_t^{'\top} \bS_t' + \frac{\lambda}{2} \bI_{d\times d}}^{-\frac{1}{2}}
    &= \sqrt{\frac{2}{\lambda}} \bI_{m\times m} + \bS_t^{'\top} \pr{\bS_t' \bS_t^{'\top}}^{-1}
      \pr{ \pr{\frac{\lambda}{2} \bI_{m\times m} + \bS_t' \bS_t^{'\top}}^{-\frac{1}{2}} - \sqrt{\frac{2}{\lambda}} \bI_{m\times m} } \bS_t'
\nonumber
\\
    &= \bS_t^{'\top} \pr{\bS_t' \bS_t^{'\top}}^{-1}
      \pr{\frac{\lambda}{2} \bI_{m\times m} + \bS_t' \bS_t^{'\top}}^{-\frac{1}{2}} \bS_t' \label{eq:sqrt_woodbury_1}
  \end{align}
  where~\eqref{eq:sqrt_woodbury_1} follows since $\bS_t^{'\top} \pr{\bS_t' \bS_t^{'\top}}^{-1} \bS_t' = \bI_{m\times m}$.
\end{proof}

\subsection{Proof of the regret bound for SOFUL (Theorem~\ref{thm:regret})}
\label{sec:regret_proofs}
We start with a preliminary lemma.
\begin{lemma}
\label{lem:inst_regret}
For any $\delta > 0$, the instantaneous regret of SOFUL satisfies
  \[
  (\bxstar_t - \bx_t)\tp \bwstar \leq 2 \tilde{\beta}_{t-1}(\delta) \|\bx_t\|_{\btilV_{t-1}^{-1}} \qquad t=1,\ldots,T~.
  \]
\end{lemma}
\begin{proof}
Let $\btilw_{t-1}\ofu$ be the FD-sketched RLS estimate of OFUL (Algorithm~\ref{alg:skoful}). Recall that the optimal action at time $t$ is $\bxstar_t = \argmax_{\bx \in D_t} \bx\tp \bwstar$, whereas
\[
	\big(\bx_t, \btilw_{t-1}\ofu\big) = \argmax_{(\bx, \bw) \in D_t \times \tilde{C}_{t-1}} \bx\tp \bw~.
\]
We use these facts to bound the instantaneous regret,
\begin{align*}
    \big(\bxstar_t - \bx_t\big)\tp \bwstar
  &\leq \bx_t\tp \btilw\ofu_{t-1} - \bx_t\tp \bwstar\\
  &= \bx_t\tp \pr{\btilw\ofu_{t-1} - \bwstar}\\
  &= \bx_t\tp \pr{\btilw\ofu_{t-1} - \btilw_{t-1}} + \bx_t\tp \pr{\btilw_{t-1} - \bwstar}\\
  &\leq \|\bx_t\|_{\btilV_{t-1}^{-1}} \Big( \|\btilw\ofu_{t-1} - \btilw_{t-1}\|_{\btilV_{t-1}} + \|\btilw_{t-1} - \bwstar\|_{\btilV_{t-1}} \Big) \tag{by Cauchy-Schwartz}\\
  &\leq 2 \tilde{\beta}_{t-1}(\delta) \|\bx_t\|_{\btilV_{t-1}^{-1}} \tag{by Theorem~\ref{thm:confidence}}
\end{align*}
concluding the proof.
\end{proof}
Now we are ready to prove the regret bound. 
\begin{proof}[Proof of Theorem~\ref{thm:regret}.]
Bounding the regret using Lemma~\ref{lem:inst_regret} gives
\begin{align*}
  R_T &= \sum_{t=1}^T \pr{\bxstar_t - \bx_t}\tp \bwstar\\
      &\leq 2 \sum_{t=1}^T \min\cbr{L S, \tilde{\beta}_{t-1}(\delta) \|\bx_t\|_{\btilV_{t-1}^{-1}}} \tag{since ${\dt \max_{t=1,\ldots,T}\max_{\bx \in D_t} |\bx\tp \bwstar| \leq L S}$ by Cauchy-Schwartz}\\
      &\leq 2 \sum_{t=1}^T \tilde{\beta}_{t-1}(\delta) \min\cbr{\frac{L}{\sqrt{\lambda}}, \|\bx_t\|_{\btilV_{t-1}^{-1}}} \tag{since ${\dt \min_{t=0,\ldots,T-1}\min_{\delta \in [0,1]} \tilde{\beta}_t(\delta) \geq S \sqrt{\lambda}}$}\\
      &\leq 2 \pr{ \max_{t=0,\ldots,T-1}\tilde{\beta}_t(\delta) } \sum_{t=1}^T \min\cbr{\frac{L}{\sqrt{\lambda}}, \|\bx_t\|_{\btilV_{t-1}^{-1}}}\\
      &\leq 2 \max\cbr{1, \frac{L}{\sqrt{\lambda}}} \pr{\max_{t=0,\ldots,T-1} \tilde{\beta}_t(\delta)} \sum_{t=1}^T \min\cbr{1, \|\bx_t\|_{\btilV_{t-1}^{-1}}}\\
      &\leq 2 \max\cbr{1, \frac{L}{\sqrt{\lambda}}} \pr{\max_{t=0,\ldots,T-1} \tilde{\beta}_t(\delta)} \sqrt{T \sum_{t=1}^T \min\cbr{1, \|\bx_t\|_{\btilV_{t-1}^{-1}}^2} } \tag{by Cauchy-Schwartz}~.
\end{align*}
Now we finish by further bounding the terms in the above. In particular, we bound $\tilde{\beta}_t(\delta)$ by~\eqref{eq:beta_tilde_bound}
\[
\max_{t=0,\ldots,T-1} \tilde{\beta}_t(\delta) \defOtilde R \sqrt{\big(m + d \ln(1 + \ve_m)\big) \pr{1 + \ve_m}} + S \sqrt{\lambda} \patch{\pr{1 + \frac{1}{\lambda}}} \pr{1 + \ve_m}
\]
while the bound on the summation term uses Lemma~\ref{lem:potential},
\[
  \sqrt{\sum_{t=1}^T \min\cbr{1, \|X_t\|_{\btilV_{t-1}^{-1}}^2}} \defOtilde \sqrt{\pr{1 + \ve_m} \big( d \ln\pr{1 + \ve_m} + m }\big)~.
\]
Then, using $M_{\lambda} = \max\cbr{1, \frac{L}{\sqrt{\lambda}}}$ and $\mtil = m + d \ln(1 + \ve_m)$,
\begin{align*}
R_T &\defOtilde M_{\lambda} \sqrt{T} \pr{R \sqrt{\mtil \pr{1 + \ve_m}} + S \sqrt{\lambda} \patch{\pr{1 + \frac{1}{\lambda}}} \pr{1 + \ve_m}} \sqrt{ \mtil \pr{1 + \ve_m} }\\
&\defOtilde M_{\lambda} \sqrt{T} \Big( R\,\mtil \pr{1 + \ve_m} + S \sqrt{\lambda} \patch{\pr{1 + \frac{1}{\lambda}}} \pr{1 + \ve_m}^{\frac{3}{2}} \sqrt{ \mtil } \Big)\\
&\defOtilde M_{\lambda} \pr{1 + \ve_m}^{\frac{3}{2}} \mtil \pr{R + S \sqrt{\lambda} \patch{\pr{1 + \frac{1}{\lambda}}}} \sqrt{T}
\end{align*}
which completes the proof.
\end{proof}
\begin{proof}[Proof of Theorem~\ref{thm:regret_star}]
Recall that
  \begin{align*}
    \Delta \leq \min_{t=1,\dots,T} \pr{\bxstar_t - \bx_t}\tp \bwstar~.
  \end{align*}
Similarly to the proof of Theorem~\ref{thm:regret}, we use Lemma~\ref{lem:inst_regret} to bound the instantaneous regret.
However, we first use the gap assumption to bound the regret in terms of the sum of squared instantaneous regrets,
  \begin{align}
    R_T &= \sum_{t=1}^T \pr{\bxstar_t - \bx_t}\tp \bwstar \nonumber\\
        &\leq \frac{1}{\Delta} \sum_{t=1}^T \pr{ \pr{\bxstar_t - \bx_t}\tp \bwstar }^2 \nonumber\\
        &\leq \frac{2}{\Delta} \sum_{t=1}^T \min\cbr{ 2 L^2 S^2, \tilde{\beta}_{t-1}(\delta)^2 \|\bx_t\|_{\btilV_{t-1}^{-1}}^2 } \label{eq:gap_regret_LS}\\
        &\leq \frac{2}{\Delta} \pr{\max_{t=0,\ldots,T-1} \tilde{\beta}_t(\delta)^2} \sum_{t=1}^T \min\cbr{ \frac{2 L^2}{\lambda}, \|\bx_t\|_{\btilV_{t-1}^{-1}}^2 } \label{eq:gap_regret_beta_sq_lb}\\
        &\leq \frac{2}{\Delta} \max\cbr{1, \frac{2 L^2}{\lambda}} \pr{\max_{t=0,\ldots,T-1} \tilde{\beta}_t(\delta)^2} \sum_{t=1}^T \min\cbr{ 1, \|\bx_t\|_{\btilV_{t-1}^{-1}}^2 } \label{eq:constituent}
  \end{align}
where~\eqref{eq:gap_regret_beta_sq_lb} holds because $\min_t\min_{\delta}\tilde{\beta}_t(\delta)^2 \geq S^2 \lambda$. Inequality~\eqref{eq:gap_regret_LS} holds because
\begin{align*}
  \pr{ \pr{\bxstar_t - \bx_t}\tp \bwstar }^2 &\leq 2 (\bxstartp_t \bwstar)^2 + 2 (\bx_t\tp \bwstar)^2
\\
                                           &\leq 4 L^2 S^2 \tag{by Cauchy-Schwartz}
\end{align*}
and because of Lemma~\ref{lem:inst_regret}.

We now finish bounding the regret by further bounding the individual terms in~\eqref{eq:constituent}. In particular, we use~\eqref{eq:beta_tilde_bound} to bound $\tilde{\beta}_t(\delta)$ as follows
\begin{align*}
\max_{t=0,\ldots,T-1} \tilde{\beta}_t(\delta)^2 &\defOtilde R^2 \pr{\sqrt{\big(m + d \ln(1 + \ve_m)\big) \pr{1 + \ve_m}} + S \sqrt{\lambda} \patch{\pr{1 + \frac{1}{\lambda}}} \pr{1 + \ve_m}}^2\\
&\defOtilde R^2 \big(m + d \ln(1 + \ve_m)\big) \pr{1 + \ve_m} + S^2 \lambda \patch{\pr{1 + \frac{1}{\lambda}}^2} \pr{1 + \ve_m}^2~.
\end{align*}
Lemma~\ref{lem:potential} gives
\[
  \sum_{t=1}^T \min\cbr{1, \|X_t\|_{\btilV_{t-1}^{-1}}^2} \defOtilde \pr{1 + \ve_m} \Big( m \ln(T) + d \ln\pr{1 + \ve_m} \Big)~.
\]
Then, using again $M_{\lambda} = \max\cbr{1, \frac{L}{\sqrt{\lambda}}}$ and $\mtil = m + d \ln(1 + \ve_m)$,
\begin{align*}
  R_T &\defOtilde \frac{M_{\lambda}^2}{\Delta} \pr{R^2 \mtil \pr{1 + \ve_m} + S^2 \lambda \patch{\pr{1 + \frac{1}{\lambda}}^2} \pr{1 + \ve_m}^2}
  \pr{1 + \ve_m} \mtil \\
      &\defOtilde \frac{M_{\lambda}^2}{\Delta} \pr{\mtil R^2 + S^2 \lambda \patch{\pr{1 + \frac{1}{\lambda}}^2}}
  \pr{1 + \ve_m}^3 \mtil \\
      &\defOtilde \frac{M_{\lambda}^2}{\Delta} \pr{R^2 + S^2 \lambda \patch{\pr{1 + \frac{1}{\lambda}}^2}}
  \pr{1 + \ve_m}^3 \mtil^2
\end{align*}
concluding the proof.
\end{proof}

\subsection{Proof of the regret bound for Sketched Linear TS (Theorem~\ref{thm:regret_ts})}
\label{sec:thompson_proof}
Here $\btilw_{t-1}\ts$ is used to denote the FD-sketched RLS estimate of linear TS (Algorithm~\ref{alg:sketched_ts}). As in~\citep{abeille2017linear}, we split the regret as follows
\begin{align}
  R_T &= \sum_{t=1}^T \big(\bxstar_t - \bx_t\big)\tp \bwstar \nonumber\\
      &= \sum_{t=1}^T \pr{ \bxstartp_t \bwstar - \bx_t\tp \btilw_{t-1}\ts } + \sum_{t=1}^T \pr{ \bx_t\tp \btilw_{t-1}\ts - \bx_t\tp \bwstar } \nonumber\\
      &= \sum_{t=1}^T \pr{ J_t(\bwstar) - J_t(\btilw_{t-1}\ts) } + \sum_{t=1}^T \pr{ \bx_t\tp \btilw_{t-1}\ts - \bx_t\tp \bwstar } \label{eq:ts_regret_decomposition}
\end{align}
where
\[
J_t(\bw) = \max_{\bx \in D_t} \bx\tp \bw
\]
is an ``optimistic'' reward function.
Most of the proof is concerned with bounding the first term in~\eqref{eq:ts_regret_decomposition}. The second term is instead obtained in way similar to the analysis of OFUL. Fix any $\delta \in (0, 1)$, let $\delta' = \frac{\delta}{4T}$, and introduce events
\begin{align*}
  \tilde{E}_t &\equiv \Big\{\|\btilw_s - \bwstar\|  \leq \tilde{\beta}_s(\delta'), \ s = 1, \ldots, t\Big\}
\\
  \tilde{E}_t\ts &\equiv \Big\{\|\btilw_s\ts - \btilw_s\| \leq \tilde{\gamma}_s(\delta'), \ s = 1, \ldots, t\Big\}
\end{align*}
and
$
E_t \equiv \tilde{E}_t \cap \tilde{E}_t\ts
$.
Observe that, by definition,
\begin{equation}
\label{eq:ts_event_inclusion}
\tilde{E}_T \subset \cdots \subset \tilde{E}_1 \qquad\text{ and }\qquad \tilde{E}_T\ts \subset \cdots \subset \tilde{E}_1\ts
\end{equation}
We also use the following lower bound on the probability of $E_T$.
\begin{lemma}
\label{lem:ts_prob_bounds}
${\dt
\P\pr{ E_T } \geq 1 - \frac{\delta}{2}
}$.
\end{lemma}
\begin{proof}
The proof is identical to the proof of \citep[Lemma~1]{abeille2017linear}, the only difference being that we use the confidence ellipsoid defined in Theorem~\ref{thm:confidence}.
\end{proof}
We study the regret when $E_T$ occurs,
\begin{align}
 \ind{E_T} R_T &= \sum_{t=1}^T \ind{E_T} \pr{ J_t(\bwstar) - J_t(\btilw_{t-1}\ts) } + \sum_{t=1}^T \ind{E_T} \pr{ \bx_t\tp \btilw_{t-1}\ts - \bx_t\tp \bwstar } \nonumber \\
                      &\leq \sum_{t=1}^T \ind{E_{t-1}} \pr{ J_t(\bwstar) - J_t(\btilw_{t-1}\ts) } + \sum_{t=1}^T \ind{E_{t-1}} \pr{ \bx_t\tp \btilw_{t-1}\ts - \bx_t\tp \bwstar } \tag{using~\eqref{eq:ts_event_inclusion}}\\
                      &= \sum_{t=1}^T r_t\ts + \sum_{t=1}^T r_t\rls  \label{eq:ts_regret_whp}
\end{align}
where we introduced the notation
\begin{align*}
  r_t\ts = \ind{E_{t-1}} \pr{ J_t(\bwstar) - J_t(\btilw_{t-1}\ts) }
\qquad\text{and}\qquad
  r_t\rls = \ind{E_{t-1}} \pr{ \bx_t\tp \btilw_{t-1}\ts - \bx_t\tp \bwstar }~.
\end{align*}
First we focus on $r_t\ts$, and get that
\begin{align*}
  r_t\ts &= \pr{J_t(\bwstar) - J_t(\btilw_{t-1}\ts)} \ind{E_{t-1}}\\
         &\leq \pr{J_t(\bwstar) - \inf_{\bw \in \tilde{C}_{t-1}\ts} J_t(\bw)} \ind{E_{t-1}} \tag{because $E_{t-1}$ implies $\btilw_{t-1}\ts \in \tilde{C}_{t-1}\ts$}\\
         &\leq \pr{J_t(\bwstar) - \inf_{\bw \in \tilde{C}_{t-1}\ts} J_t(\bw)} \ind{\tilde{E}_{t-1}}~. \tag{using \eqref{eq:ts_event_inclusion}}
\end{align*}
Consider the following set of ``optimistic'' coefficients $\bw$ such that $J_t(\bwstar) \leq J_t(\bw)$ and, moreover, $\bw$ belongs to the sketched TS confidence ellipsoid,
\[
W\optts_t \equiv \cbr{\bw \in \reals^d ~:~ J_t(\bwstar) \leq J_t(\bw)} \cap \tilde{C}_t\ts~.
\]
Then, for $\btilw\ts \in W_{t-1}\optts$
\begin{equation}
  \label{eq:ts_rt_expectation}
  r_t\ts \leq \pr{J_t(\btilw\ts) - \inf_{\bw \in \tilde{C}_{t-1}\ts} J_t(\bw)} \ind{\tilde{E}_{t-1}}~.
\end{equation}
We now use \citep[Proposition~3 and Lemma~2]{abeille2017linear} (restated below here for convenience) to argue about the convexity of $J$ and relate its gradient to the chosen action.
\begin{proposition}
  \label{prop:ts_J_is_convex}
  For any finite set $D$ of actions $\bx$ such that $\norm{\bx}\le 1$, $\max_{\bx \in D} \bx\tp \bw$ is convex on $\reals^d$. Moreover, it is continuous with continuous first derivatives (except for a zero-measure set w.r.t.\ the Lebesgue measure).
\end{proposition}
\begin{lemma}
  \label{lem:ts_gradient_to_action}
  For any $\bw \in \reals^d$, we have
\[
	\nabla \Big( \max_{\bx \in D} \bx\tp\bw \Big) = \argmax_{\bx\in D} \bx\tp\bw
\]
(except for a zero-measure w.r.t.\ the Lebesgue measure).
\end{lemma}
Relying on the two results above, we can proceed as follows.
Introduce $J_t^{/L}(\bw) = J_t(\bw) / L = \max_{\bx \in D_t} (\bx/L)\tp \bw$.
Then by Proposition~\ref{prop:ts_J_is_convex}, $J_t^{/L}(\bw)$ is convex for $\bw \in \reals^d$ since $\|\bx/L\| \leq 1$.
Then, by letting $\bxstar(\btilw\ts) = \nabla J_t(\btilw\ts)$, for any $\btilw\ts \in W_{t-1}\optts$ we have
\begin{align*}
  J_t(\btilw\ts) - \inf_{\bw \in \tilde{C}_{t-1}\ts} J_t(\bw)
  &= L \pr{J_t^{/L}(\btilw\ts) - \inf_{\bw \in \tilde{C}_{t-1}\ts} J_t^{/L}(\bw)}\\
  &\leq L \sup_{\bw \in \tilde{C}_{t-1}\ts}\cbr{ \nabla J_t^{/L}(\btilw\ts)\tp \pr{\btilw\ts - \bw} }\\
&= L \sup_{\bw \in \tilde{C}_{t-1}\ts}\cbr{ \pr{\frac{\bxstar(\btilw\ts)}{L}} \tp \pr{\btilw\ts - \bw} } \\
&\leq \|\bxstar(\btilw\ts)\|_{\btilV_{t-1}^{-1}} \sup_{\bw \in \tilde{C}_{t-1}\ts} \|\btilw\ts - \bw\|_{\btilV_{t-1}} \tag{by Cauchy-Schwartz}\\
&\leq 2 \tilde{\gamma}_{t-1}(\delta') \|\bxstar(\btilw\ts)\|_{\btilV_{t-1}^{-1}}
\end{align*}
where the last inequality holds for all $\btilw\ts \in \tilde{C}_{t-1}\ts$ and by the triangle inequality. Substituting this into~\eqref{eq:ts_rt_expectation}, and taking expectation with respect to $\btilw\ts$ yields
\begin{equation}
\label{eq:ts_r_t_bounded_by_cond_E}
  r_t\ts \leq 2 \tilde{\gamma}_{t-1}(\delta') \E\left[\|\bxstar(\btilw\ts)\|_{\btilV_{t-1}^{-1}}  \ind{\tilde{E}_{t-1}} \,\middle|\, \btilw\ts \in W\optts_{t-1},\,\sF_{t-1}\right]~.
\end{equation}
where we use $\sF_t$ to denote the $\sigma$-algebra generated by the random variables $\eta_1,\boldeta_1,\dots,\eta_{t-1},\boldeta_{t-1}$.
Now we further upper bound $r_t\ts$ while bounding the probability of event $\btilw\ts \in W_{t-1}\optts$ occurring in~\eqref{eq:ts_r_t_bounded_by_cond_E}. This is done in the following lemma, whose proof (omitted here) is identical to the proof of \cite[Lemma~3]{abeille2017linear}, where ellipsoids are replaced by their sketched counterparts.
\begin{lemma}
\label{lem:ts_p_lb}
Assume that $\sD\ts$ is a TS-sampling distribution with anti-concentration parameter $p$. Then, for $\boldeta \sim \sD\ts$ we have that
  \[
  \P\pr{\btilw\ts \in W_{t-1}\optts \,\middle|\, \tilde{E}_{t-1}, \sF_{t-1}} \geq \frac{p}{2} \qquad t=1,\ldots,T~.
  \]
\end{lemma}
We now proceed with the main argument of the proof. Using $g(\btilw\ts) = \|\bxstar(\btilw\ts)\|_{\btilV_{t-1}^{-1}}$,
\begin{align*}
	\E\br{ g(\btilw\ts)  \,\middle|\, \tilde{E}_{t-1}, \sF_{t-1}  }
&\geq
	\E\br{ g(\btilw\ts) \ind{\btilw\ts \in W\optts_{t-1}}  \,\middle|\, \tilde{E}_{t-1}, \sF_{t-1} }
\\&=
	\E\br{ g(\btilw\ts)  \,\middle|\, \btilw\ts \in W\optts_{t-1}, \tilde{E}_{t-1}, \sF_{t-1} }
	\P\pr{\btilw\ts \in W\optts_{t-1} \,\middle|\, \tilde{E}_{t-1}, \sF_{t-1}}
\\&\ge
	\E\br{ g(\btilw\ts)  \,\middle|\, \btilw\ts \in W\optts_{t-1}, \tilde{E}_{t-1}, \sF_{t-1} } \frac{p}{2} \tag{by Lemma~\ref{lem:ts_p_lb}.}
\end{align*}
The above combined with~\eqref{eq:ts_r_t_bounded_by_cond_E} implies that
\begin{align}
	  r_t\ts
&\le
	2 \tilde{\gamma}_{t-1}(\delta') \E\left[g(\btilw\ts) \ind{\tilde{E}_{t-1}} \,\middle|\, \btilw\ts \in W\optts_{t-1}, \sF_{t-1}\right]
\nonumber
\\&=
\nonumber
	2 \tilde{\gamma}_{t-1}(\delta') \E\left[g(\btilw\ts) \,\middle|\, \btilw\ts \in W\optts_{t-1}, \tilde{E}_{t-1}, \sF_{t-1}\right] \P\big(\tilde{E}_{t-1}\big)
\nonumber
\\&\le
	\frac{4}{p} \tilde{\gamma}_{t-1}(\delta') \E\br{ g(\btilw\ts) \,\middle|\, \tilde{E}_{t-1}, \sF_{t-1} }~. \label{eq:ts_bound_on_r_ts}
\end{align}
Finally, summing~\eqref{eq:ts_bound_on_r_ts} over time we get
\begin{align*}
\sum_{t=1}^T r_t\ts &\leq \frac{4}{p} \left(\max_{t=0,\ldots,T}\cbr{\tilde{\gamma}_t(\delta')}\right) \sum_{t=1}^T \E\br{ \|\bxstar(\btilw\ts)\|_{\btilV_{t-1}^{-1}}  \ \middle| \ \sF_{t-1}  }~.
\end{align*}
Note that we can already bound $\tilde{\gamma}_t$ using~\eqref{eq:gamma_tilde_bound}. However, we cannot bound the expectation right away, so we rewrite the above as follows
\begin{equation}
\label{eq:ts_rts_bound}
\sum_{t=1}^T r_t\ts \leq \frac{4}{p} \left(\max_{t=0,\ldots,T}\cbr{\tilde{\gamma}_t(\delta')}\right) \pr{ \sum_{t=1}^T \|X_t\|_{\btilV_{t-1}^{-1}} + M_T }
\end{equation}
where we introduce the martingale
\[
M_T = \sum_{t=1}^T \pr{ \E\br{ \|\bxstar(\btilw\ts)\|_{\btilV_{t-1}^{-1}} \ \middle| \ \sF_{t-1}  } - \|X_t\|_{\btilV_{t-1}^{-1}} }~.
\]
Next, we use the Azuma-Hoeffding inequality to upper-bound $M_T$.
\begin{theorem}[Azuma-Hoeffding inequality]
If a supermartingale $Y_t$ corresponding to a filtration $\sF_t$ satisfies $|Y_t - Y_{t-1}| \leq c_t$ for some constant $c_t$ for $t=1,2,\ldots$, then for any $\alpha$,
\[
\P\pr{Y_T - Y_0 \geq \alpha} \leq \exp\left(- \frac{\alpha^2}{2 \sum_{t=1}^T c_t^2}\right)~.
\]
\end{theorem}
Now verify that for any $t=1,\ldots,T$,
\[
M_t - M_{t-1} = \E\br{ \|\bxstar(\btilw\ts)\|_{\btilV_{t-1}^{-1}} \ \middle| \ \sF_{t-1}  } - \|X_t\|_{\btilV_{t-1}^{-1}} \leq \frac{2 L}{\sqrt{\lambda}}~.
\]
Thus, by the Azuma-Hoeffding inequality, with probability at least $1 - \delta/2$ we have
\begin{equation}
\label{eq:ts_bound_M_T}
M_T \leq \sqrt{\frac{4 L T}{\lambda} \ln\pr{\frac{4}{\delta}}}~.
\end{equation}
Now we focus our attention on the remaining term:
\begin{align}
  \sum_{t=1}^T \|X_t\|_{\btilV_{t-1}^{-1}} &\leq \sum_{t=1}^T \min\cbr{\frac{L}{\sqrt{\lambda}}, \|X_t\|_{\btilV_{t-1}^{-1}}} \nonumber\\
                                        &\leq \max\cbr{1, \frac{L}{\sqrt{\lambda}}} \sum_{t=1}^T \min\cbr{1, \|X_t\|_{\btilV_{t-1}^{-1}}} \nonumber\\
                                        &\leq \max\cbr{1, \frac{L}{\sqrt{\lambda}}} \sqrt{T \sum_{t=1}^T \min\cbr{1, \|X_t\|_{\btilV_{t-1}^{-1}}^2}} \tag{by Cauchy-Schwartz} \nonumber\\
  &\defOtilde \max\cbr{1, \frac{1}{\sqrt{\lambda}}} \sqrt{(1 + \ve_m) \pr{d \ln(1 + \ve_m) + m} T} \label{eq:ts_potential_bound}
\end{align}
where the last step is due to Lemma~\ref{lem:potential}.

For brevity denote $\mtil = m + d \ln(1 + \ve_m)$. Now, we substitute into~\eqref{eq:ts_rts_bound} the bound~\eqref{eq:ts_bound_M_T} on $M_T$, the bound~\eqref{eq:ts_potential_bound}, and the bound~\eqref{eq:gamma_tilde_bound} on $\tilde{\gamma}_t$. This gives
\begin{align}
  \sum_{t=1}^T r_t\ts &\defOtilde \sqrt{d} \pr{ R \sqrt{\mtil \pr{1 + \ve_m}} + S \sqrt{\lambda} \patch{\pr{1 + \frac{1}{\lambda}}} \pr{1 + \ve_m} }
\pr{ \max\cbr{1, \frac{1}{\sqrt{\lambda}}} \sqrt{(1 + \ve_m) \mtil T} + \sqrt{\frac{T}{\lambda}} } \nonumber\\
                      &\defOtilde \max\cbr{1, \frac{1}{\sqrt{\lambda}}} \mtil \pr{1 + \ve_m}^{\frac{3}{2}} \pr{ R + S \sqrt{\lambda} \patch{\pr{1 + \frac{1}{\lambda}}} } \sqrt{dT} \label{eq:ts_rts_final_bound}
\end{align}
which holds with high probability (due to Azuma-Hoeffding inequality).

Now we bound the remaining RLS term of the regret. In particular,
\begin{align}
\sum_{t=1}^T r_t\rls &= \sum_{t=1}^T \ind{E_{t-1}} \pr{ X_t\tp \btilw_{t-1}\ts - X_t\tp \bwstar } \nonumber \\
  &= \sum_{t=1}^T \ind{E_{t-1}} \pr{ X_t\tp \btilw_{t-1}\ts -  X_t\tp \btilw_{t-1}}
    + \sum_{t=1}^T \ind{E_{t-1}} \pr{ X_t\tp \btilw_{t-1} - X_t\tp \bwstar } \nonumber \\
  &\leq \sum_{t=1}^T \ind{E_{t-1}} \|X_t\|_{\btilV_{t-1}} \|\btilw_{t-1}\ts -  \btilw_{t-1}\|_{\btilV_{t-1}^{-1}} \nonumber \\
    &+ \sum_{t=1}^T \ind{E_{t-1}} \|X_t\|_{\btilV_{t-1}} \|\btilw_{t-1} - \bwstar\|_{\btilV_{t-1}^{-1}} \tag{by Cauchy-Schwartz} \nonumber \\
  &\leq \sum_{t=1}^T \|X_t\|_{\btilV_{t-1}} \tilde{\gamma}_{t-1}(\delta') \tag{by definition of event $\tilde{E}_{t-1}\ts$} \nonumber \\
    &+ \sum_{t=1}^T \|X_t\|_{\btilV_{t-1}} \tilde{\beta}_{t-1}(\delta') \tag{by definition of event $\tilde{E}_{t-1}$} \nonumber \\
  &\defOtilde \max\cbr{1, \frac{1}{\sqrt{\lambda}}} \sqrt{\mtil (1 + \ve_m) T} \tag{using~\eqref{eq:ts_potential_bound}}\\
  &\cdot d \pr{ R \sqrt{\mtil \pr{1 + \ve_m}} + S \sqrt{\lambda} \patch{\pr{1 + \frac{1}{\lambda}}} \cdot \pr{1 + \ve_m} } \tag{using Theorem~\ref{thm:confidence} to bound $\tilde{\beta}$ and~\eqref{eq:gamma_tilde_bound} to bound $\tilde{\gamma}$}\\
  &\defOtilde \max\cbr{1, \frac{1}{\sqrt{\lambda}}}
    \pr{ R \mtil \pr{1 + \ve_m} + S \sqrt{\lambda} \patch{\pr{1 + \frac{1}{\lambda}}} \sqrt{\mtil} \pr{1 + \ve_m}^{\frac{3}{2}} } \sqrt{dT} \nonumber \\
  &\defOtilde \max\cbr{1, \frac{1}{\sqrt{\lambda}}} \mtil \pr{1 + \ve_m}^{\frac{3}{2}}
    \pr{ R + S \sqrt{\lambda} \patch{\pr{1 + \frac{1}{\lambda}}} } \sqrt{dT}~. \label{eq:ts_rls_final_bound}
\end{align}
Hence, combining~\eqref{eq:ts_regret_whp}, \eqref{eq:ts_rts_final_bound}, and~\eqref{eq:ts_rls_final_bound} gives, with high probability,
\begin{align*}
  \ind{E_T} R_T = \sum_{t=1}^T r_t\ts + \sum_{t=1}^T r_t\rls
\defOtilde \max\cbr{1, \frac{1}{\sqrt{\lambda}}} \mtil \pr{1 + \ve_m}^{\frac{3}{2}} \pr{ R + S \sqrt{\lambda} \patch{\pr{1 + \frac{1}{\lambda}}} } \sqrt{dT}
\end{align*}
The proof is concluded by observing that Lemma~\ref{lem:ts_prob_bounds} proves that $E_T$ also holds with high probability.

\section{Experiments}
In this section we present experiments on six publicly available classification datasets.
\paragraph{Setup.}
The idea of our experimental setup is similar to the one described by~\cite{cesa2013gang}. Namely, we convert a $K$-class classification problem into a contextual bandit problem as follows: given a dataset of labeled instances $(\bx,y)\in\reals^d\times\{1,\dots,K\}$, we partition it into $K$ subsets according to the class labels. Then we create $K$ sequences by drawing a random permutation of each subset. At each step $t$ the decision set $D_t$ is obtained by picking the $t$-th instance from each one of these $K$ sequences. Finally, rewards are determined by choosing a class $y \in \{1,\dots,K\}$ and then consistently assigning reward $1$ to all instances labeled with $y$ and reward $0$ to all remaining instances.
\paragraph{Datasets.}
We perform experiments on six publicly available datasets for multiclass classification from the \verb!openml! repository~\citep{openml} ---dataset IDs 1461, 23, 32, 182, 22, and 44, see the table below here for details.
\begin{center}
\begin{tabular}{llrr}
Dataset & Examples & Features & Classes \\
\hline
Bank & 45k & 17 & 2 \\
SatImage & 6k & 37 & 6 \\
Spam & 4k & 58 & 2 \\
Pendigits & 11k & 17 & 10 \\
MFeat & 2k & 48 & 10 \\
CMC & 1.4k & 10 & 3 \\
\end{tabular}
\end{center}
\paragraph{Baselines.}
The hyperparameters $\beta$ (confidence ellipsoid radius) and $\lambda$ (RLS regularization parameter) are selected on a validation set of size $100$ via grid search on $(\beta, \lambda) \in \cbr{1, 10^2, 10^3, 10^4} \times \cbr{10^{-2}, 10^{-1}, 1}$ for OFUL, and $\cbr{1, 10^2, 10^3} \times \cbr{10^{-2}, 10^{-1}, 1, 10^2}$ for linear TS.
\paragraph{Results}
We observe that on three datasets, Figure~\ref{fig:appendix:cumulative_reward}, sketched algorithms indeed do not suffer a substantial drop in performance when compared to the non-sketched ones, even when the sketch size amounts to $60\%$ of the context space dimension. This demonstrates that sketching successfully captures relevant subspace information relatively to the goal of maximizing reward.

Because the FD-sketching procedure considered in this paper is essentially performing online PCA, it is natural to ask how our sketched algorithms would compare to their non-sketched version run on the best $m$-dimensional subspace (computed by running PCA on the entire dataset). In Figure~\ref{fig:appendix:pca_oful}, we compare SOFUL and sketched linear TS to their non-sketched versions. In particular, we keep $60\%, 40\%$, and $20\%$ of the top principal components, and notice that, like in Figure~\ref{fig:appendix:cumulative_reward}, there are cases with little or no loss in performance.
\begin{figure*}
  \centering
  \includegraphics[width=5cm]{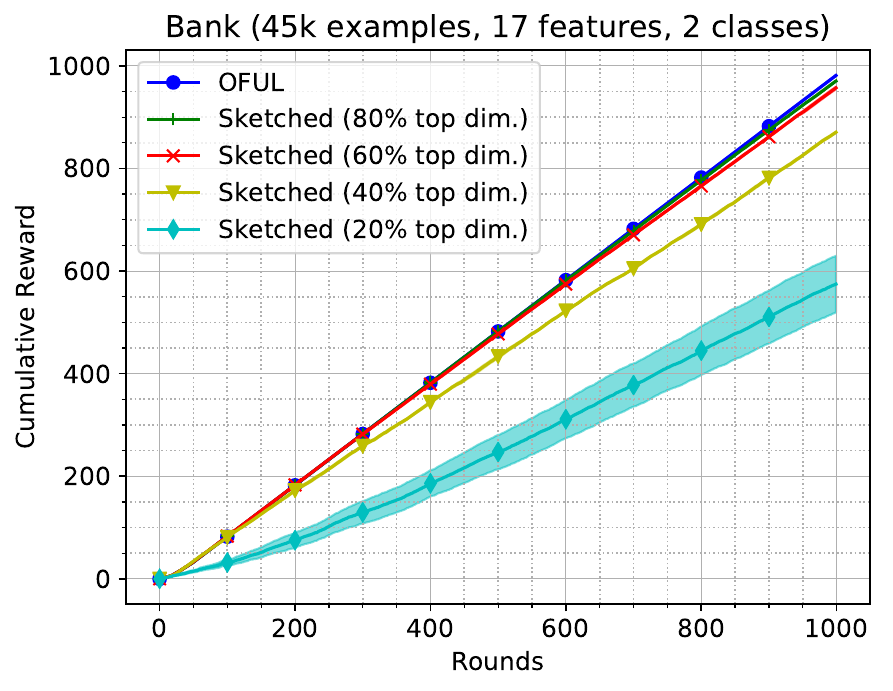}
  \includegraphics[width=5cm]{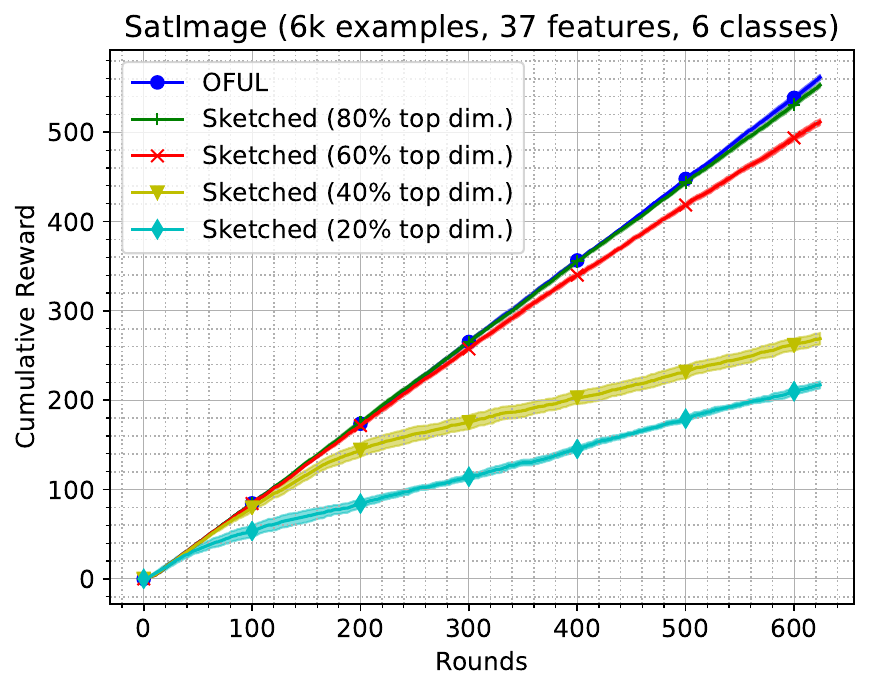}
  \includegraphics[width=5cm]{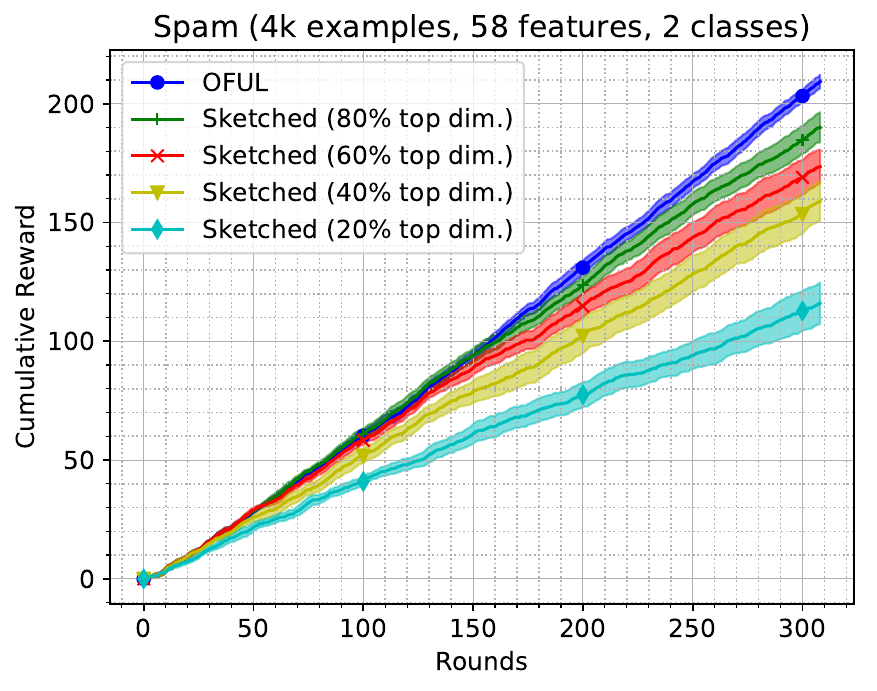}\\
  \includegraphics[width=5cm]{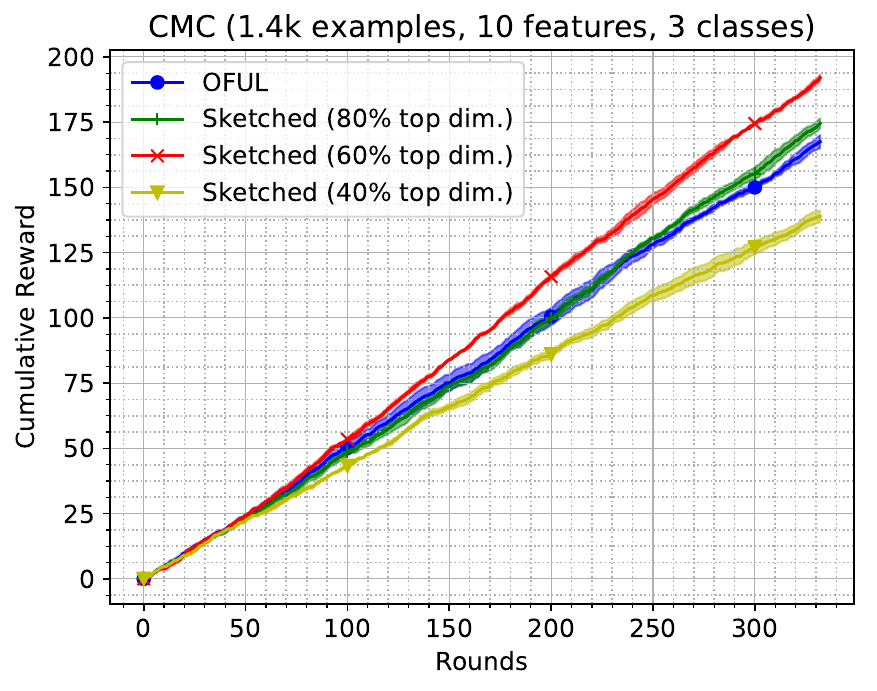}
  \includegraphics[width=5cm]{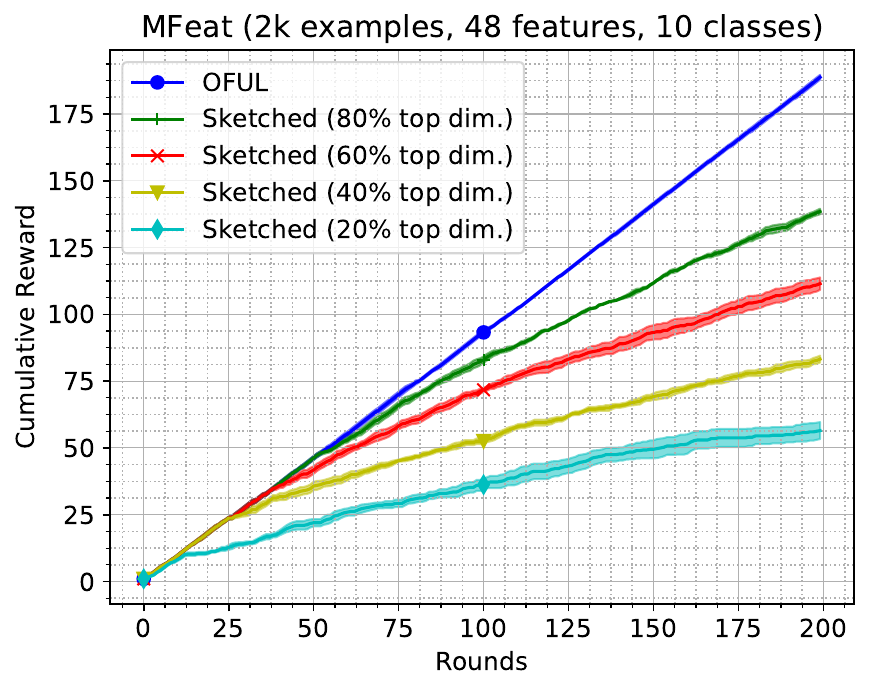}
  \includegraphics[width=5cm]{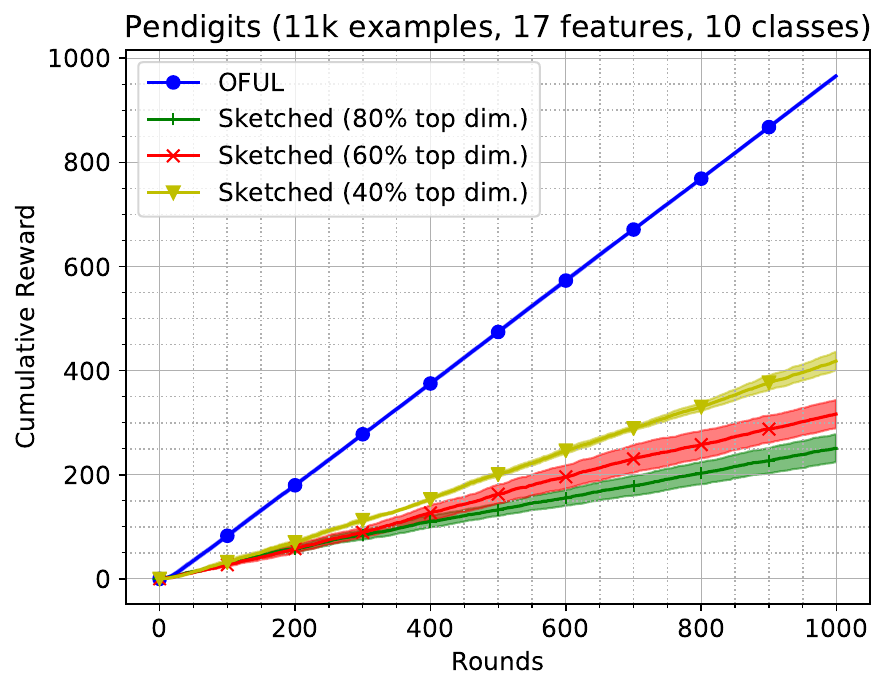}
  \caption{Comparison of SOFUL to OFUL on six real-world datasets and for different sketch sizes.
    Note that, in some cases, a sketch size equal to $80\%$ and even $60\%$ of the context space dimension does not significantly affect the perfomance.}
\end{figure*}
\begin{figure*}
  \ContinuedFloat
  \centering
  \includegraphics[width=5cm]{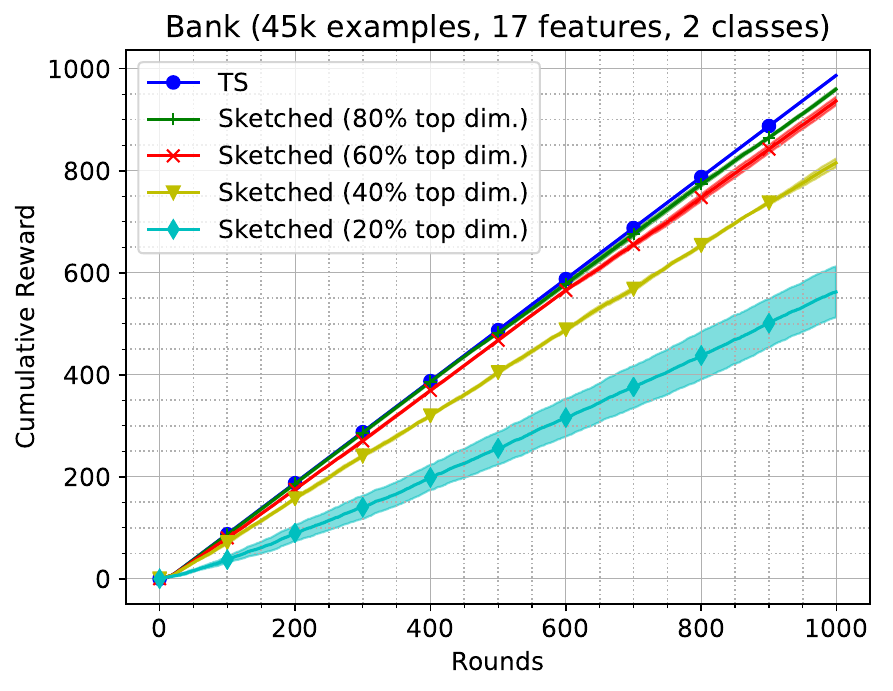}
  \includegraphics[width=5cm]{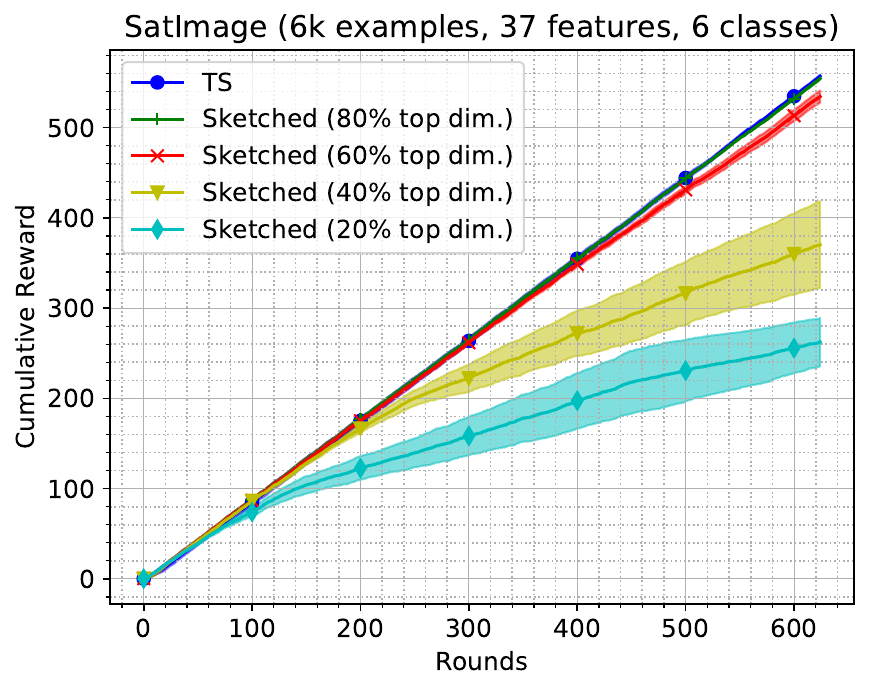}
  \includegraphics[width=5cm]{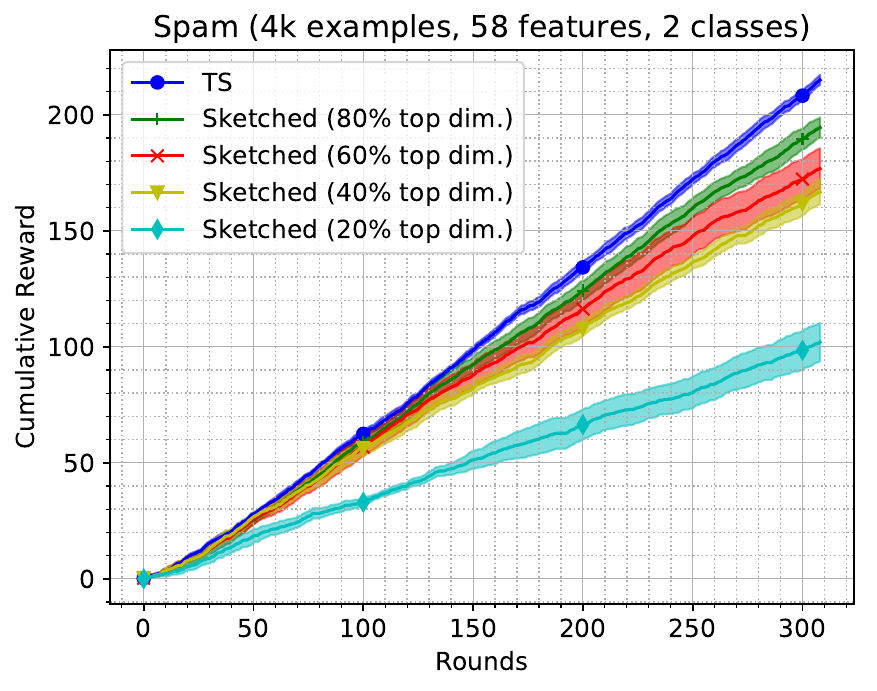}\\
  \includegraphics[width=5cm]{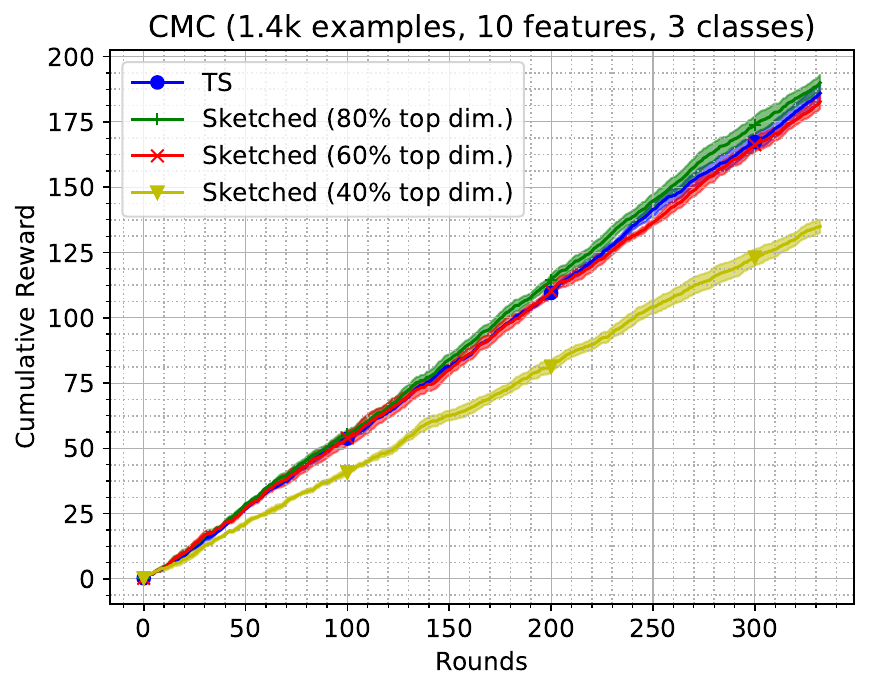}
  \includegraphics[width=5cm]{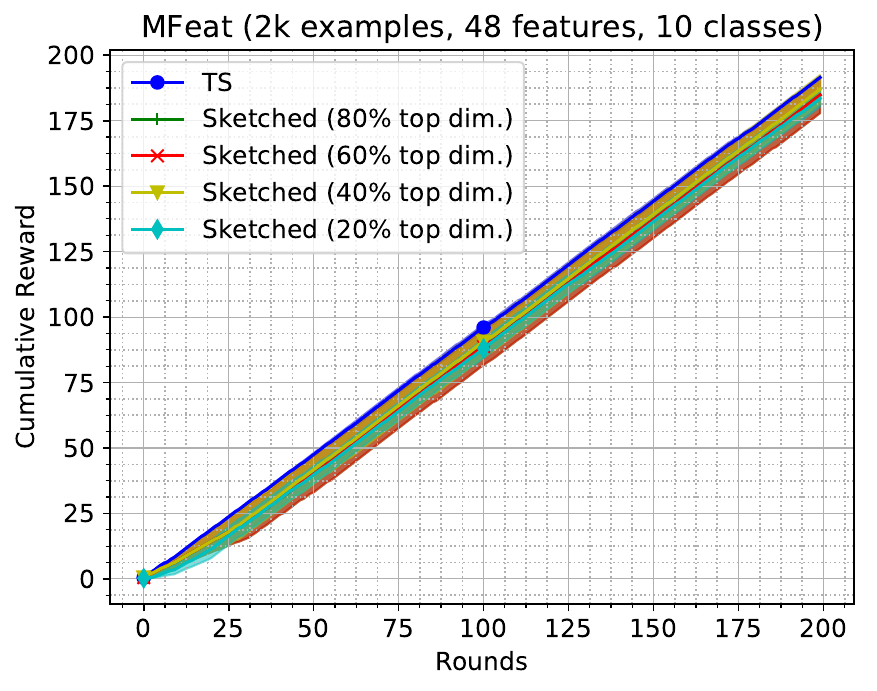}
  \includegraphics[width=5cm]{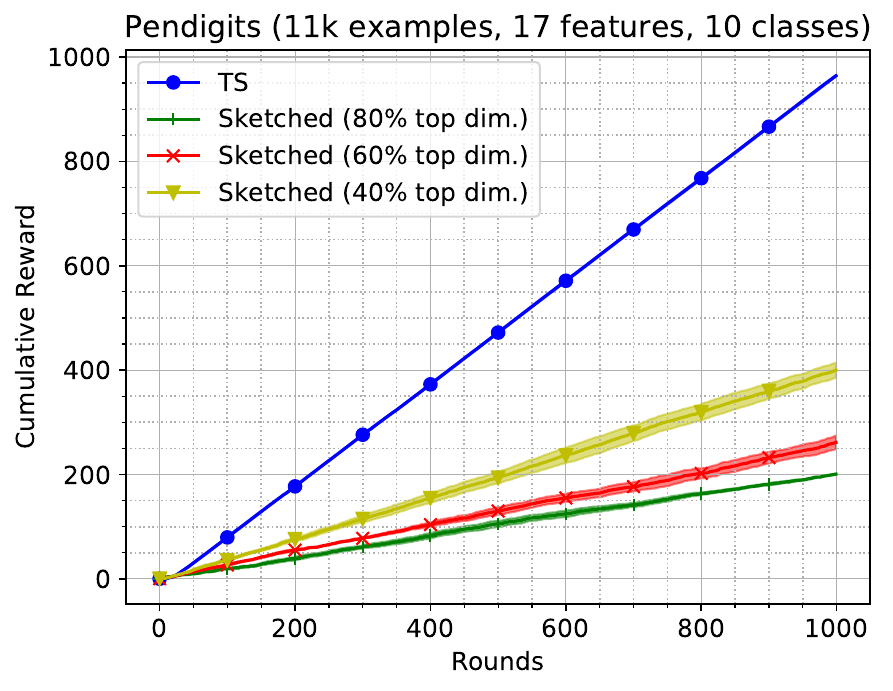}
  \caption{Comparison of sketched linear TS to linear TS on six real-world datasets and for different sketch sizes.
    Note that, in some cases, a sketch size equal to $80\%$ and even $60\%$ of the context space dimension does not significantly affect the perfomance.}
\label{fig:appendix:cumulative_reward}
\end{figure*}
\begin{figure*}
  \centering
  \includegraphics[width=4cm]{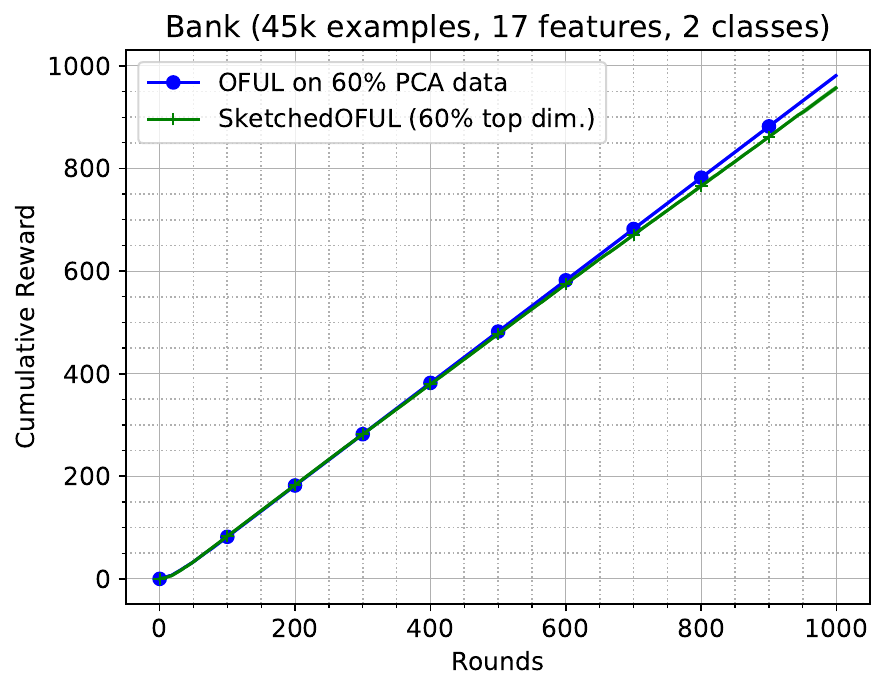}
  \includegraphics[width=4cm]{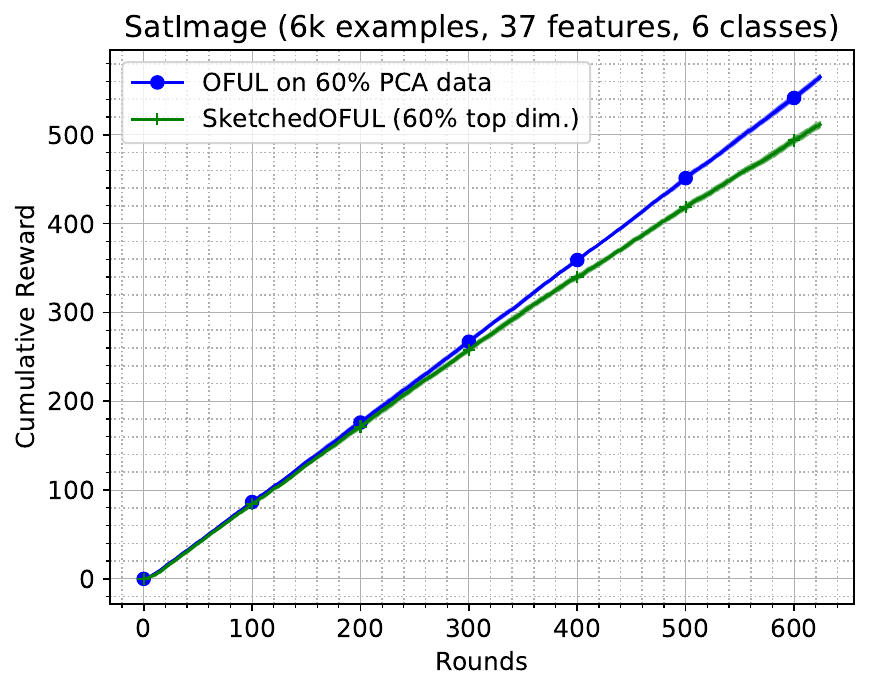}
  \includegraphics[width=4cm]{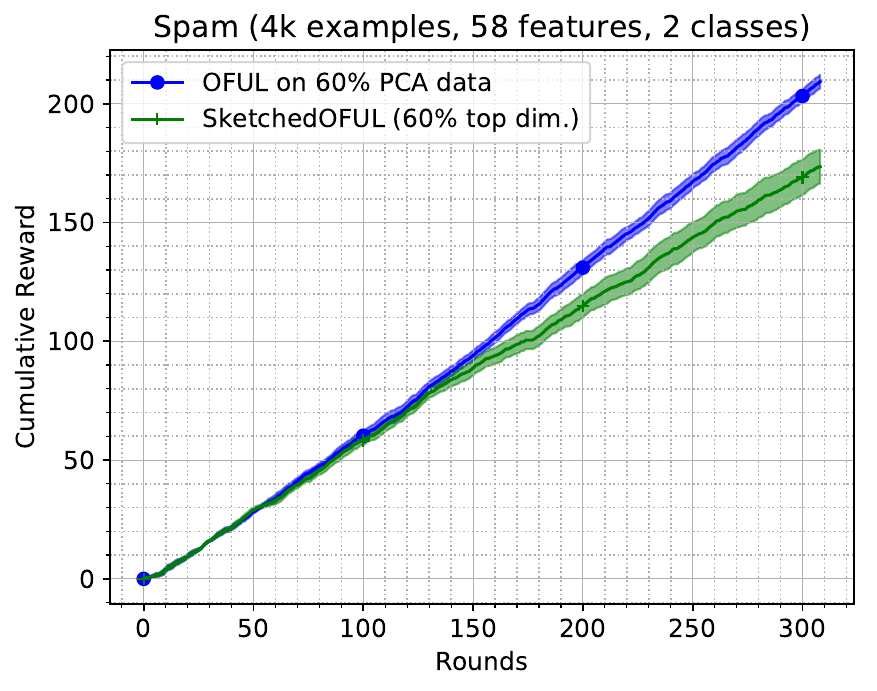}\\
  \includegraphics[width=4cm]{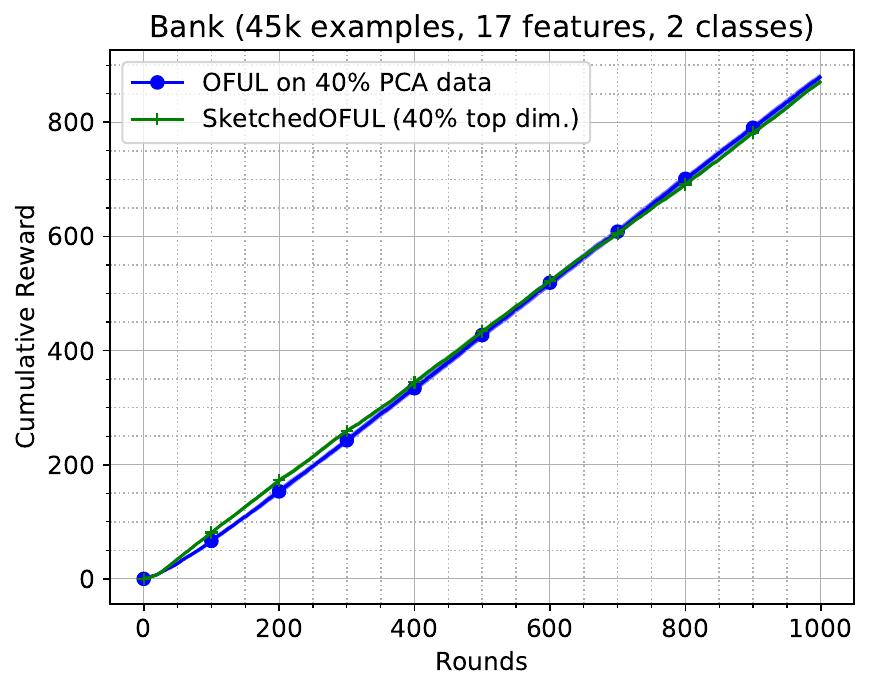}
  \includegraphics[width=4cm]{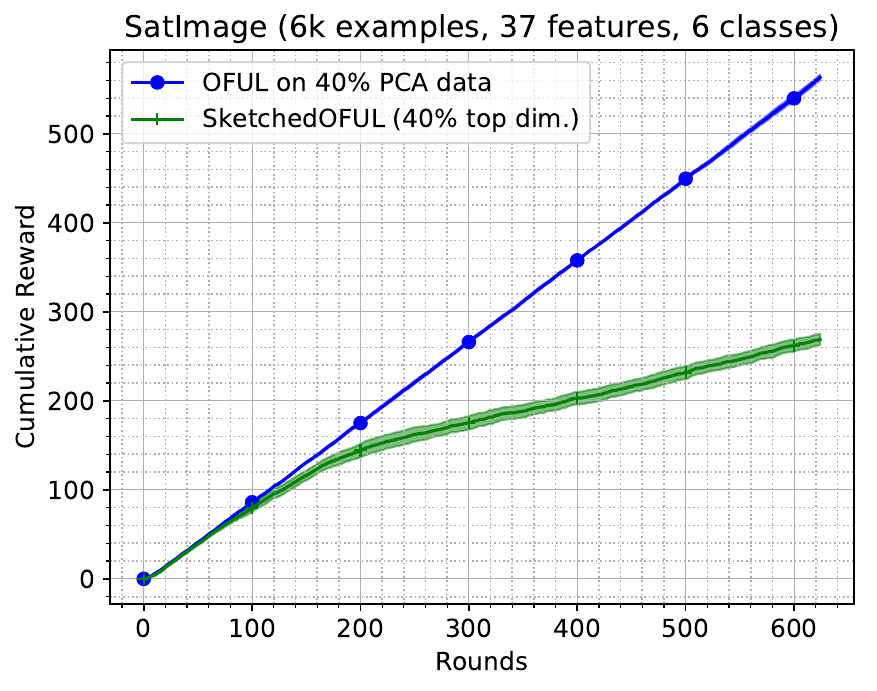}
  \includegraphics[width=4cm]{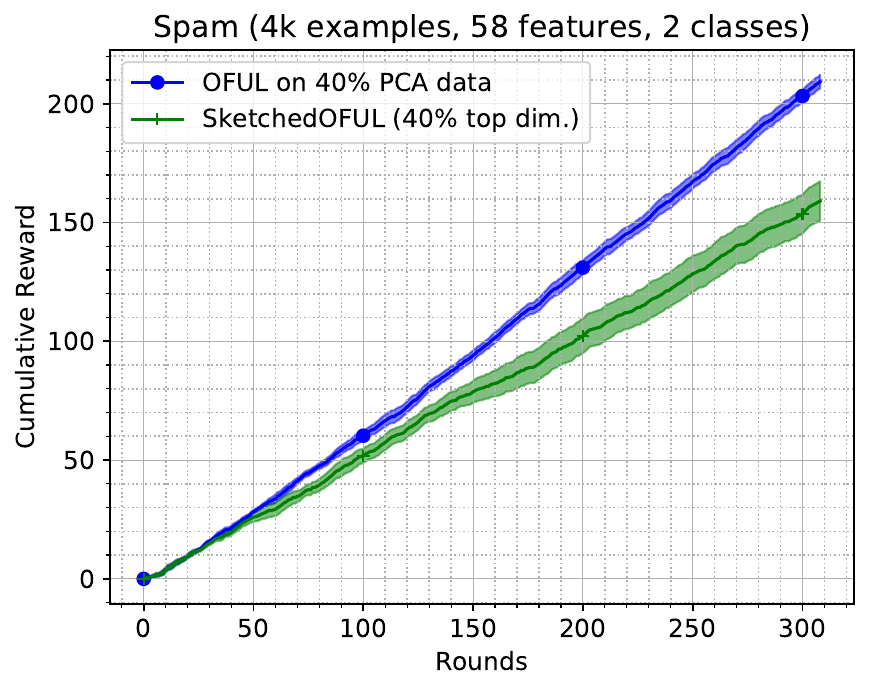}\\
  \includegraphics[width=4cm]{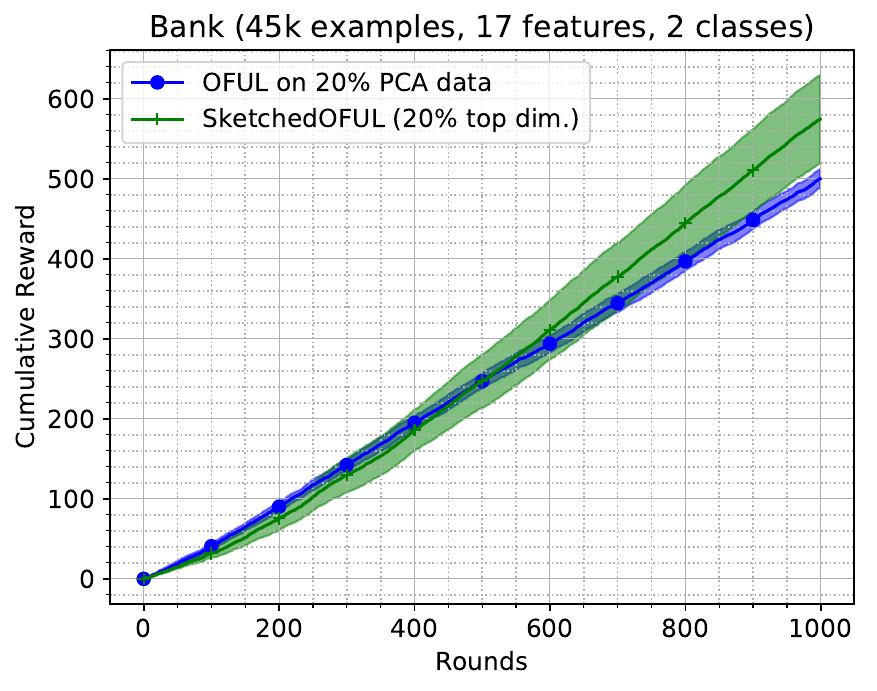}
  \includegraphics[width=4cm]{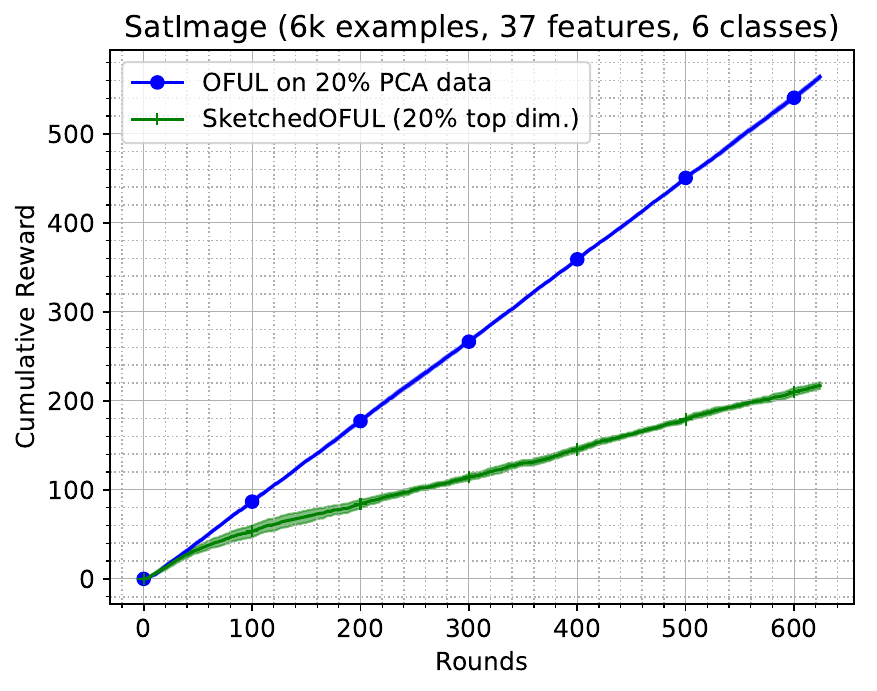}
  \includegraphics[width=4cm]{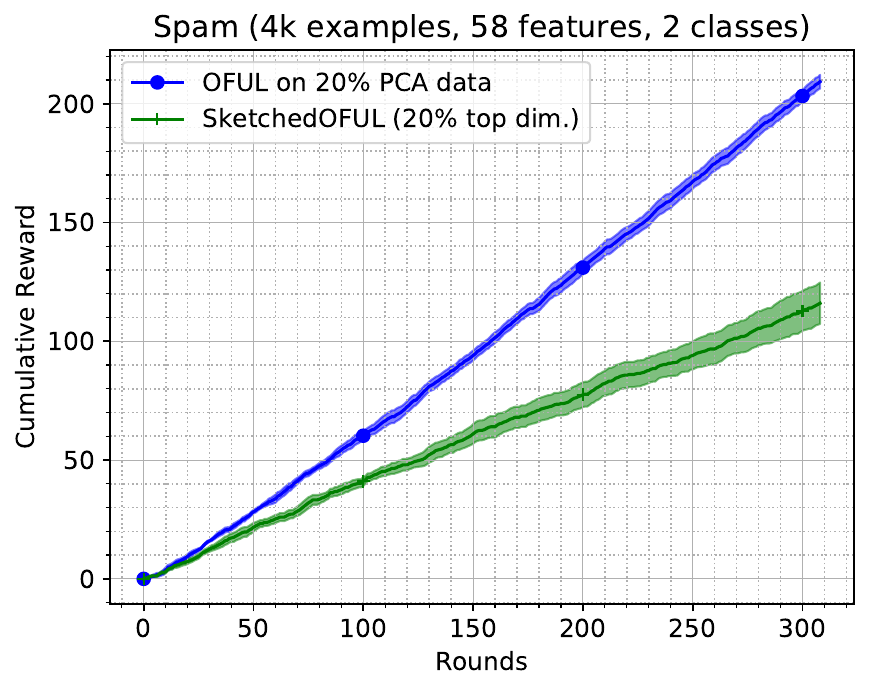}
  ~\\
  {\color[rgb]{0.75,0.75,0.75} \par\noindent\rule{0.75\textwidth}{0.4pt}}
  ~\\~\\
  \includegraphics[width=4cm]{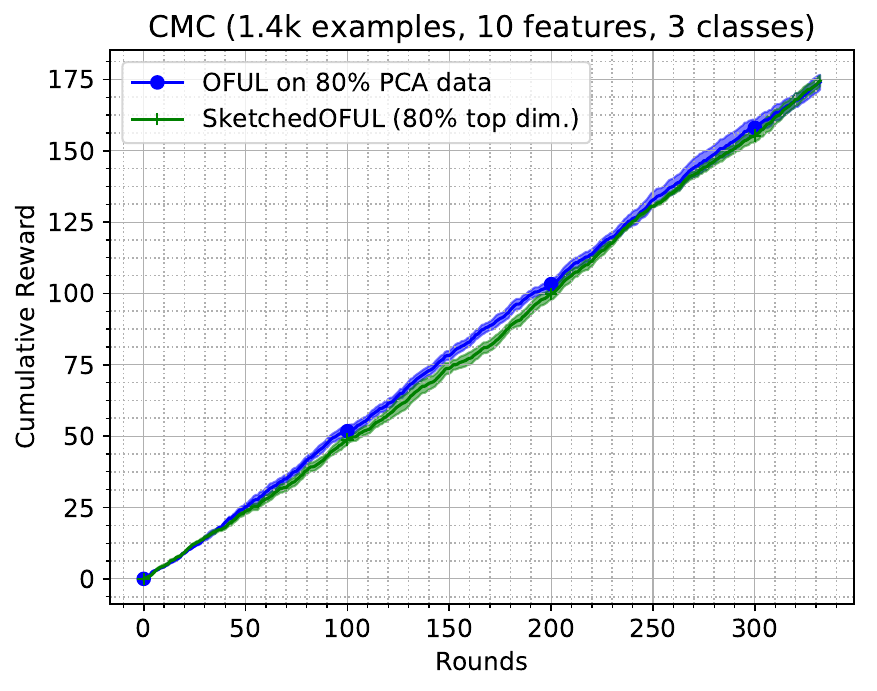}
  \includegraphics[width=4cm]{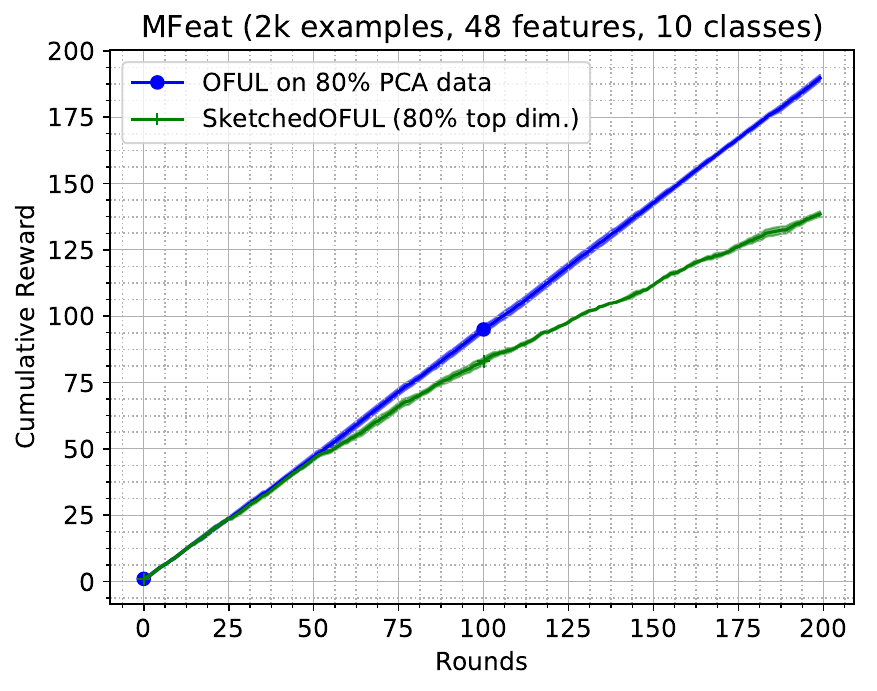}
  \includegraphics[width=4cm]{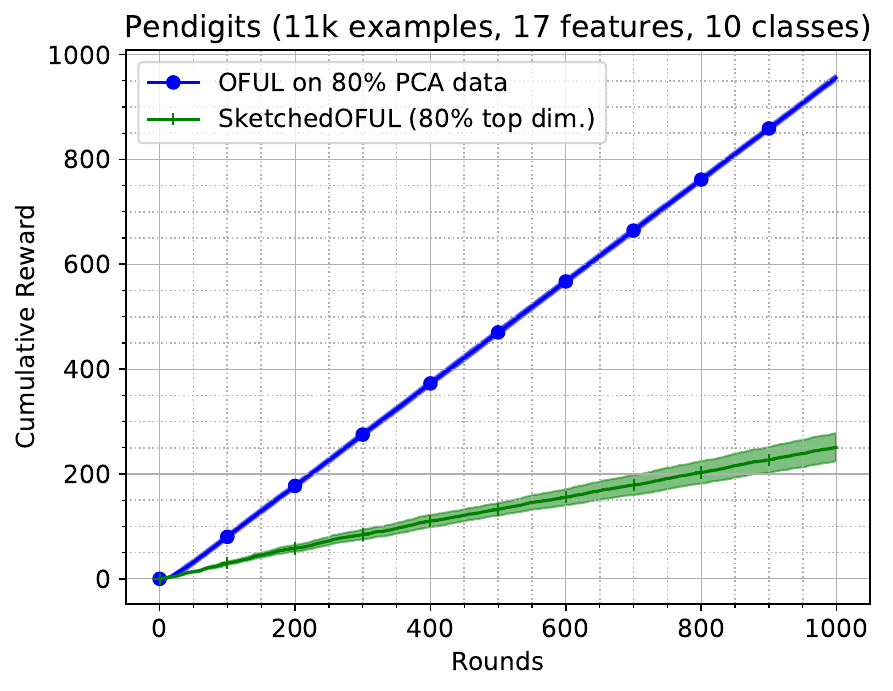}\\
  \includegraphics[width=4cm]{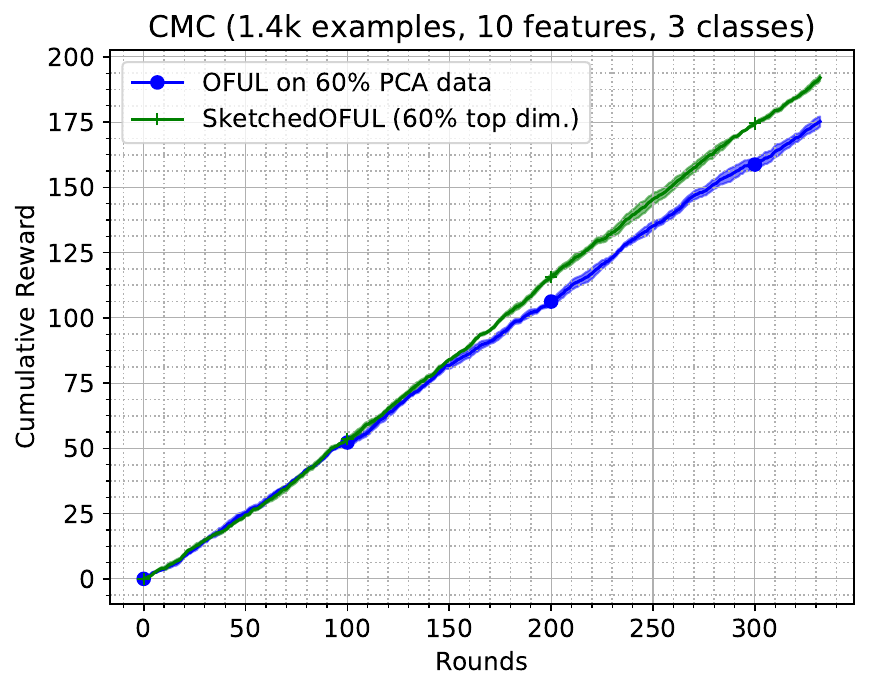}
  \includegraphics[width=4cm]{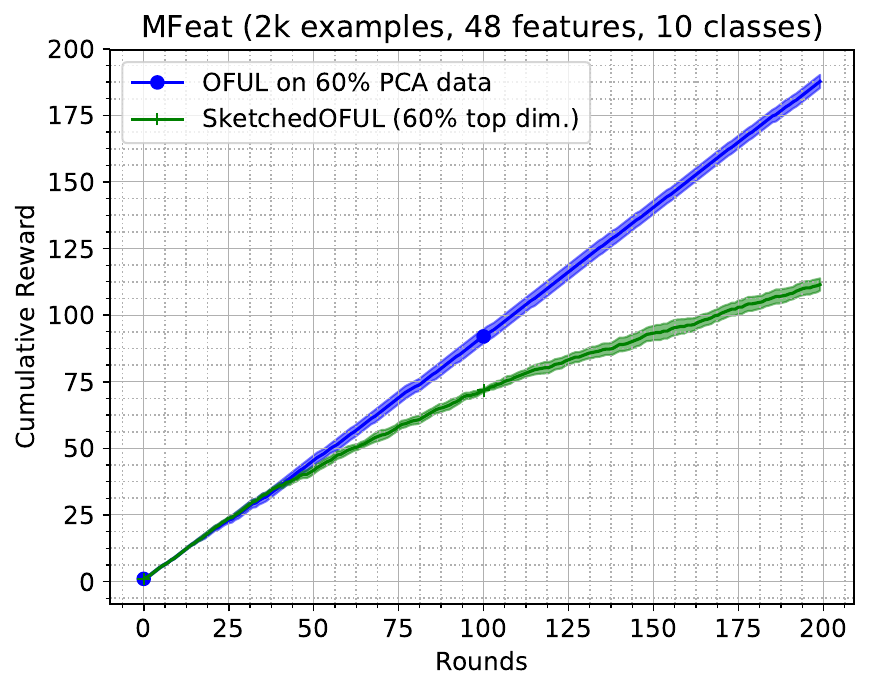}
  \includegraphics[width=4cm]{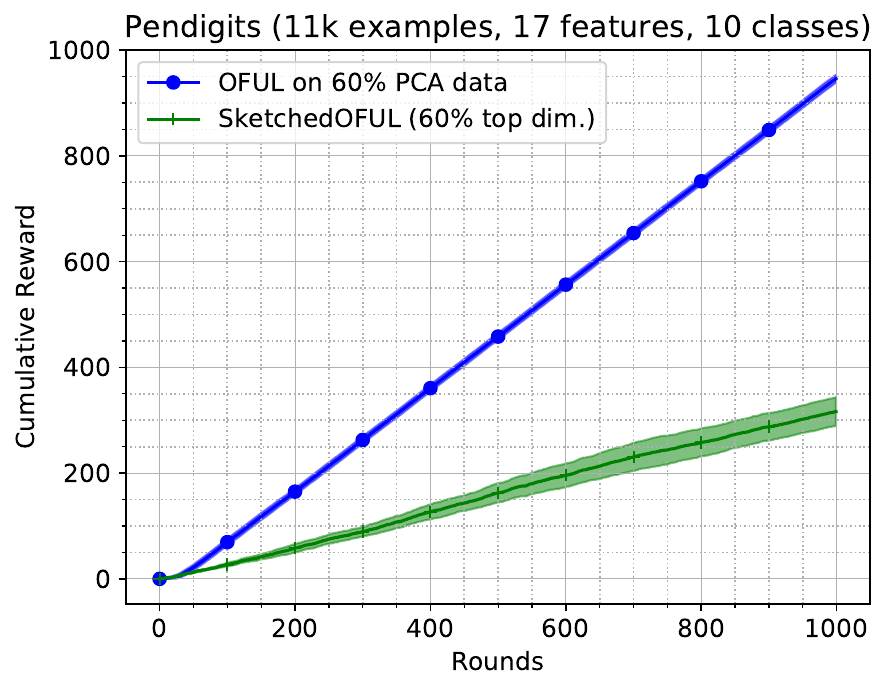}\\
  \includegraphics[width=4cm]{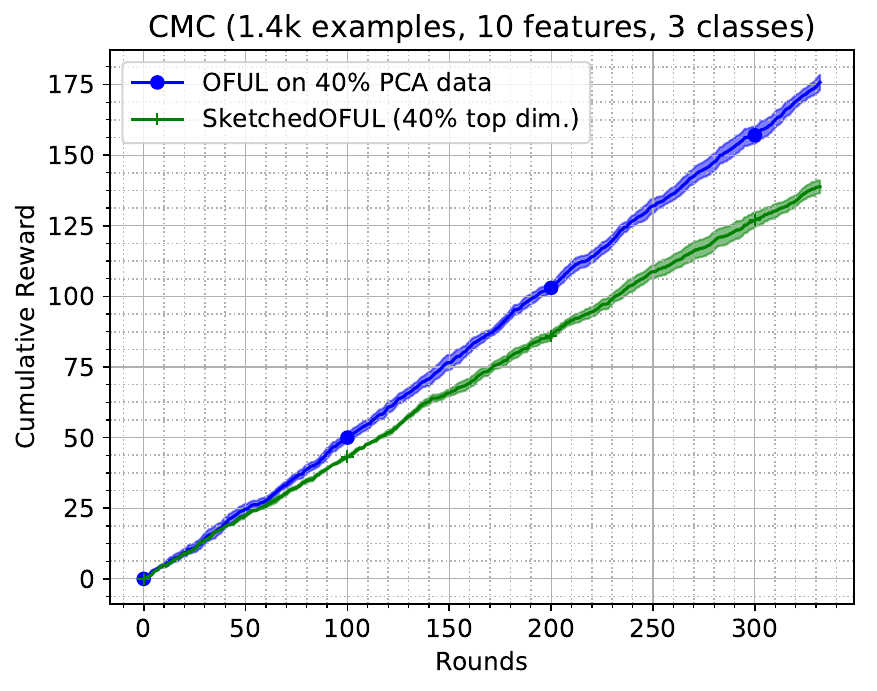}
  \includegraphics[width=4cm]{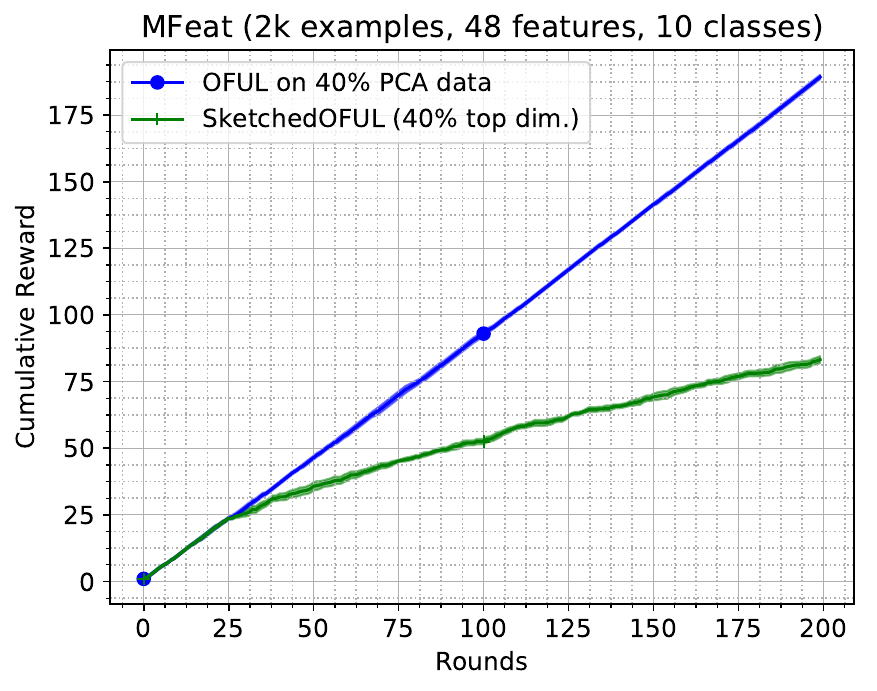}
  \includegraphics[width=4cm]{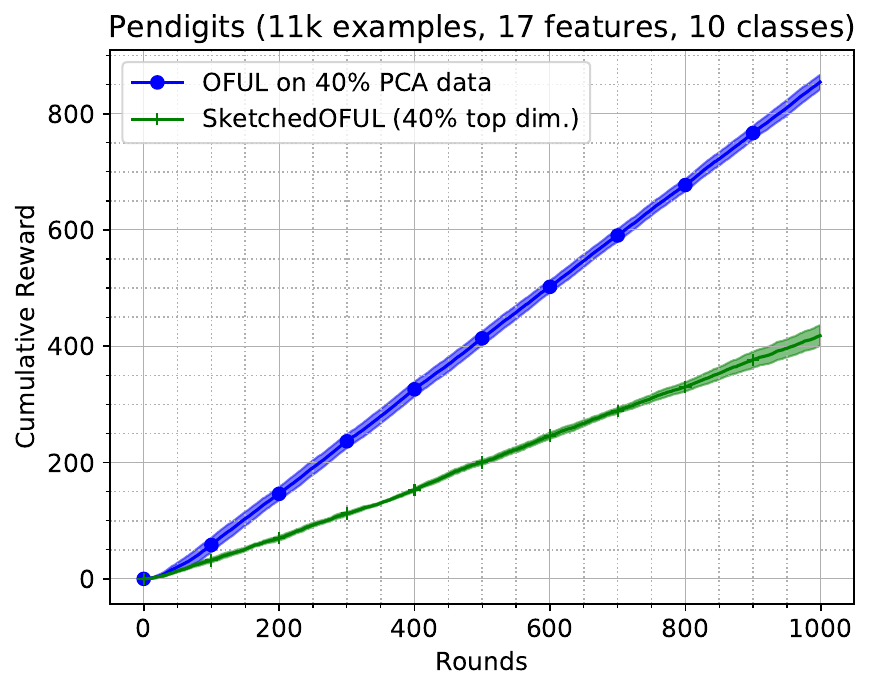}
\caption{Comparison of OFUL run on the best $m$-dimensional subspace against SOFUL run with sketch size $m$. Rows show $m$ as a fraction of the context space dimension: $60\%, 40\%, 20\%$ (for the first three datasets), while columns correspond to different datasets.
  Note that, in some cases (with sketch size $m$ of size at least $60\%$), SOFUL performs as well as if the best $m$-dimensional subspace had been known in hindsight.}
\label{fig:appendix:pca_oful}
\end{figure*}
\begin{figure*}
  \centering
  \includegraphics[width=4cm]{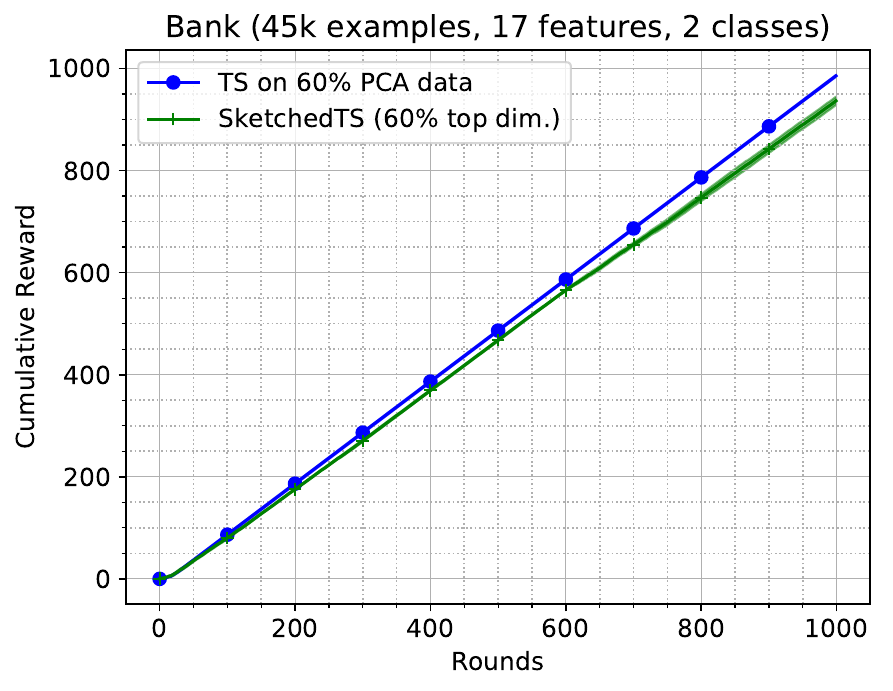}
  \includegraphics[width=4cm]{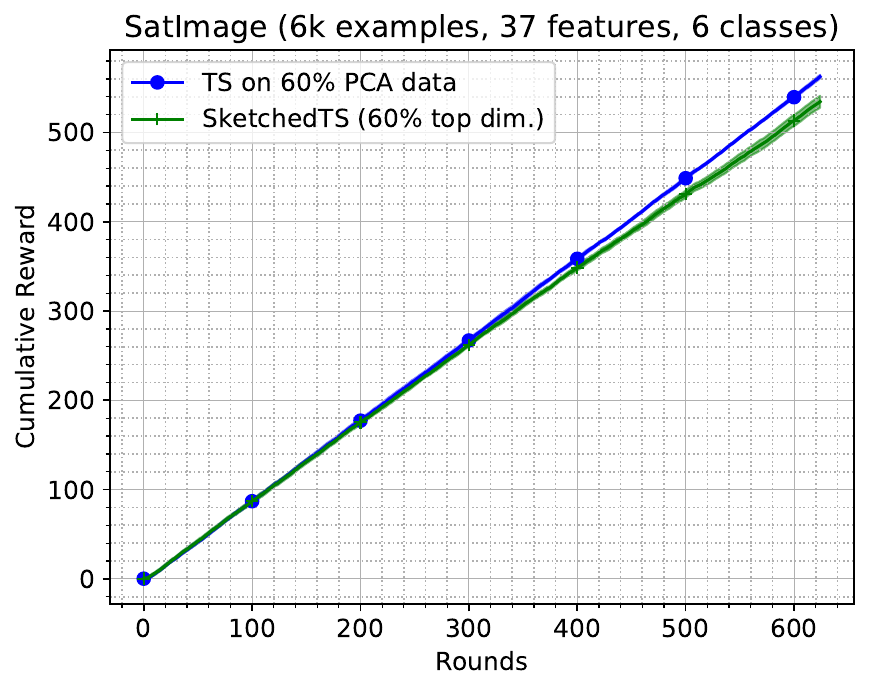}
  \includegraphics[width=4cm]{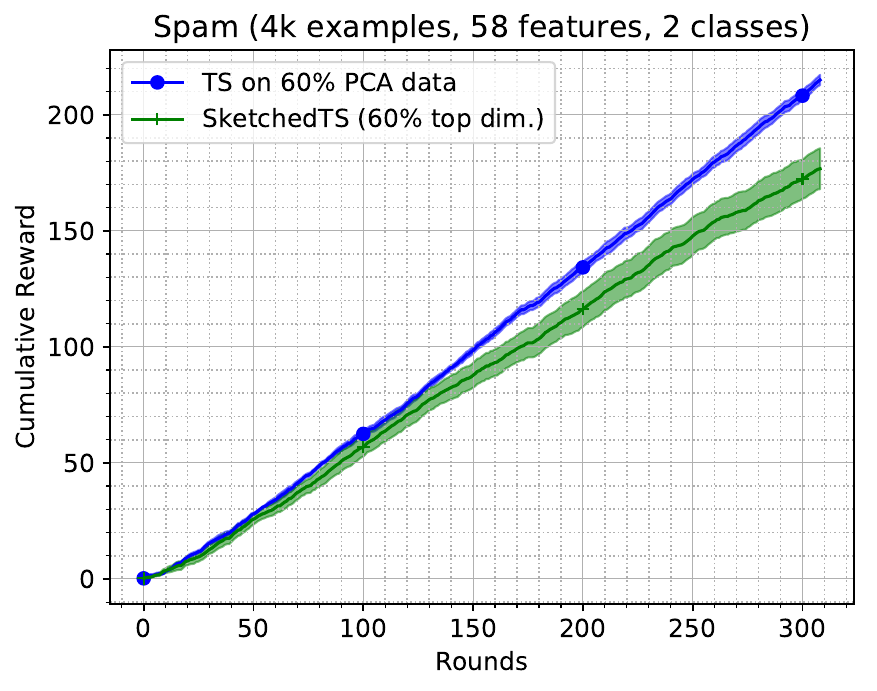}\\
  \includegraphics[width=4cm]{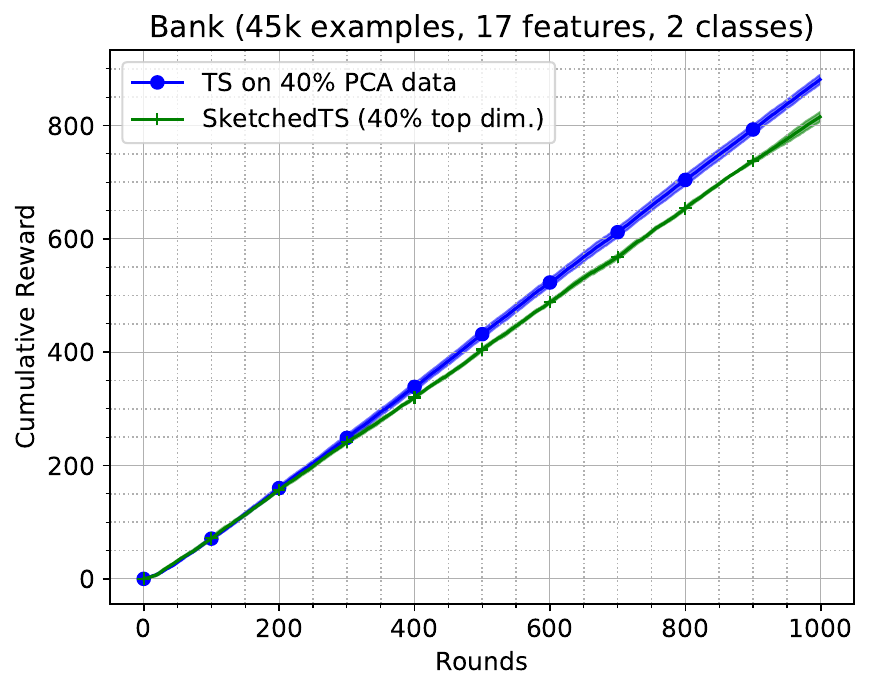}
  \includegraphics[width=4cm]{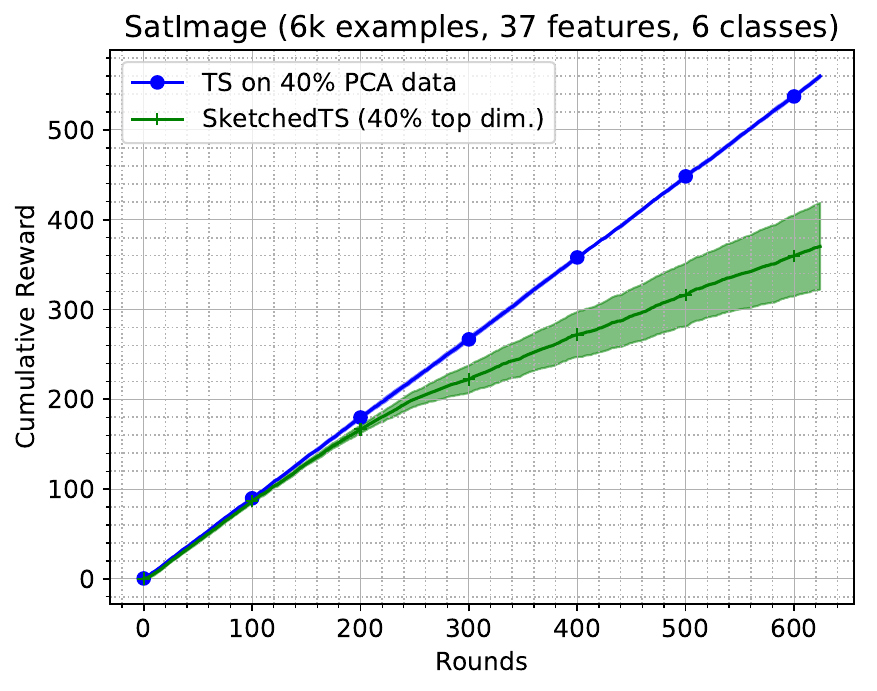}
  \includegraphics[width=4cm]{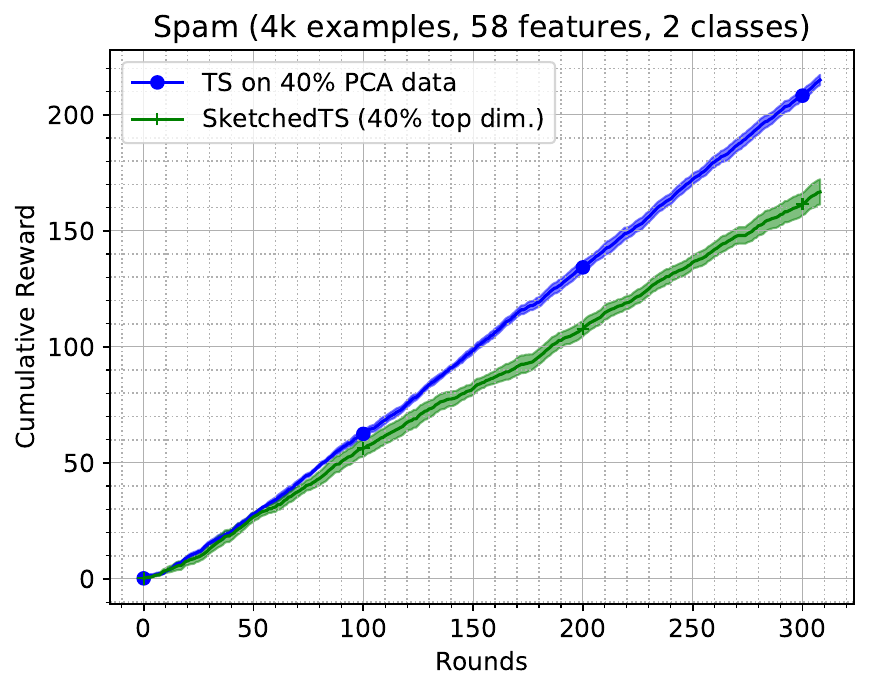}\\
  \includegraphics[width=4cm]{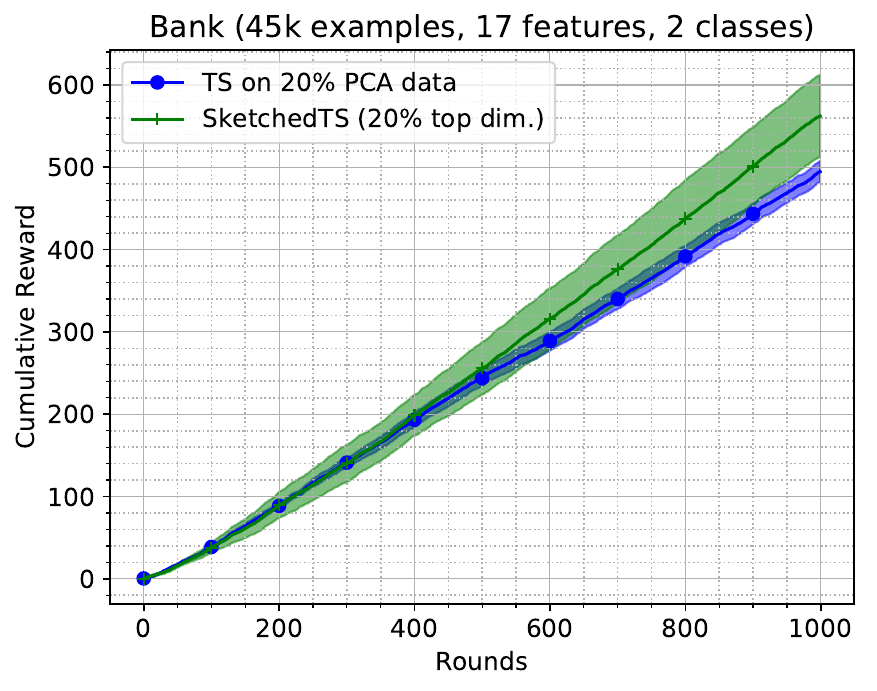}
  \includegraphics[width=4cm]{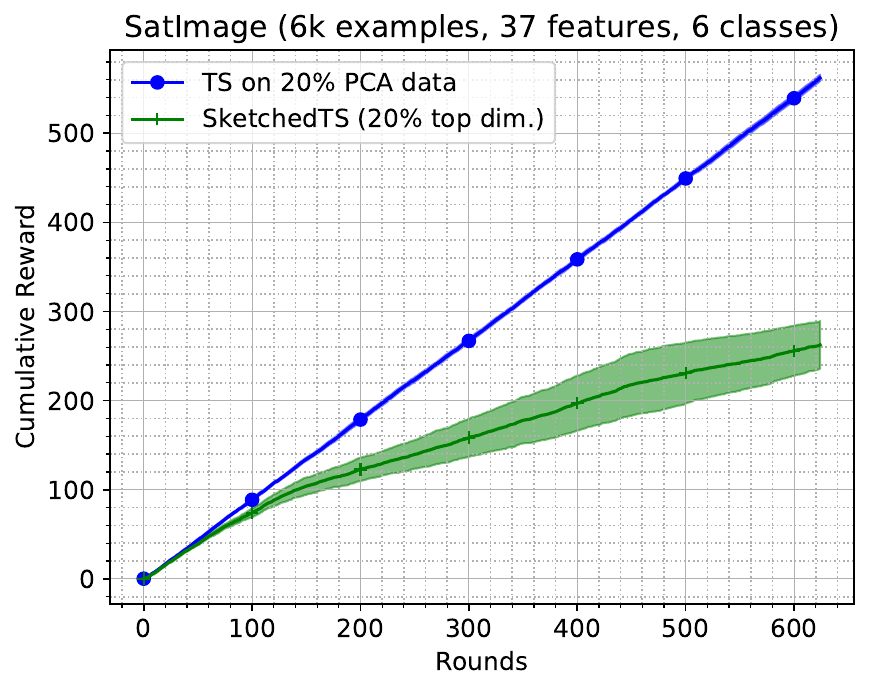}
  \includegraphics[width=4cm]{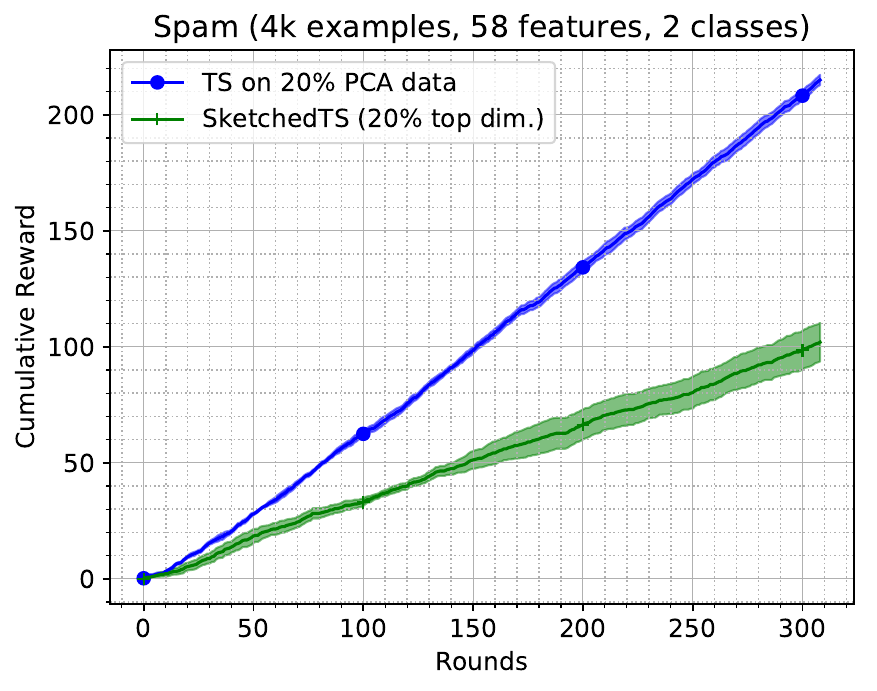}
  ~\\
  {\color[rgb]{0.75,0.75,0.75} \par\noindent\rule{0.75\textwidth}{0.4pt}}
  ~\\~\\
  \includegraphics[width=4cm]{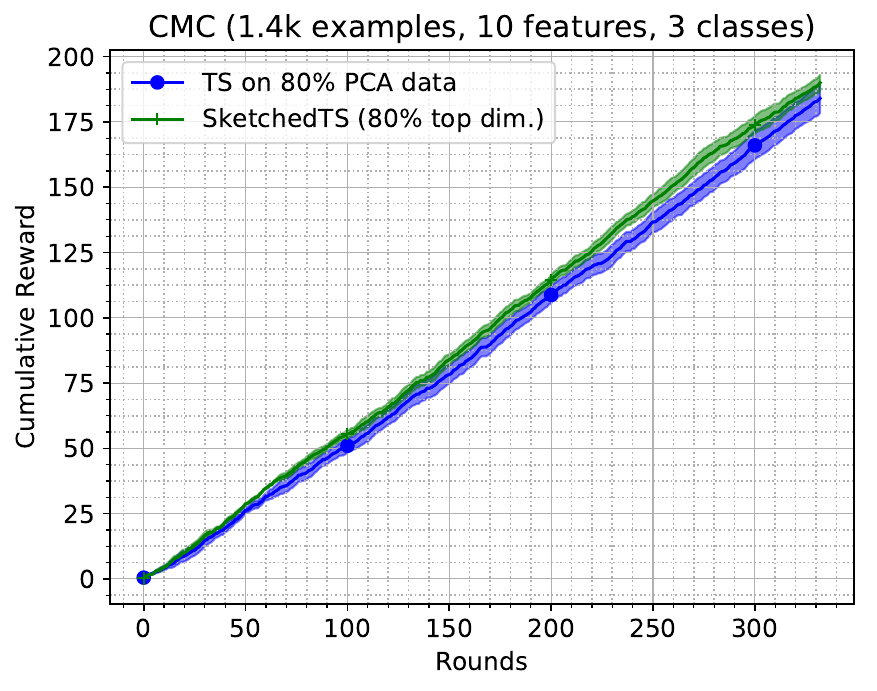}
  \includegraphics[width=4cm]{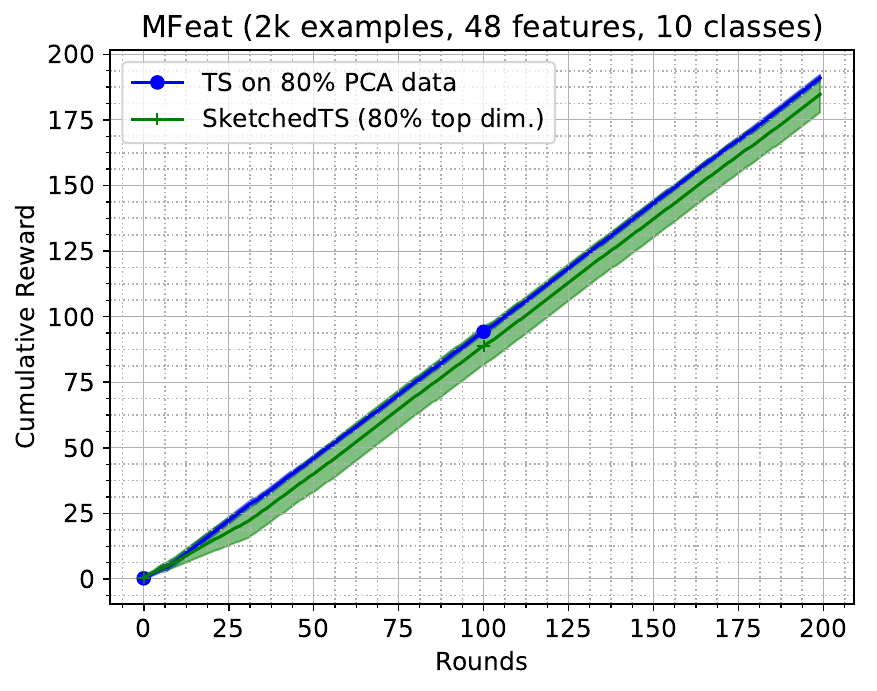}
  \includegraphics[width=4cm]{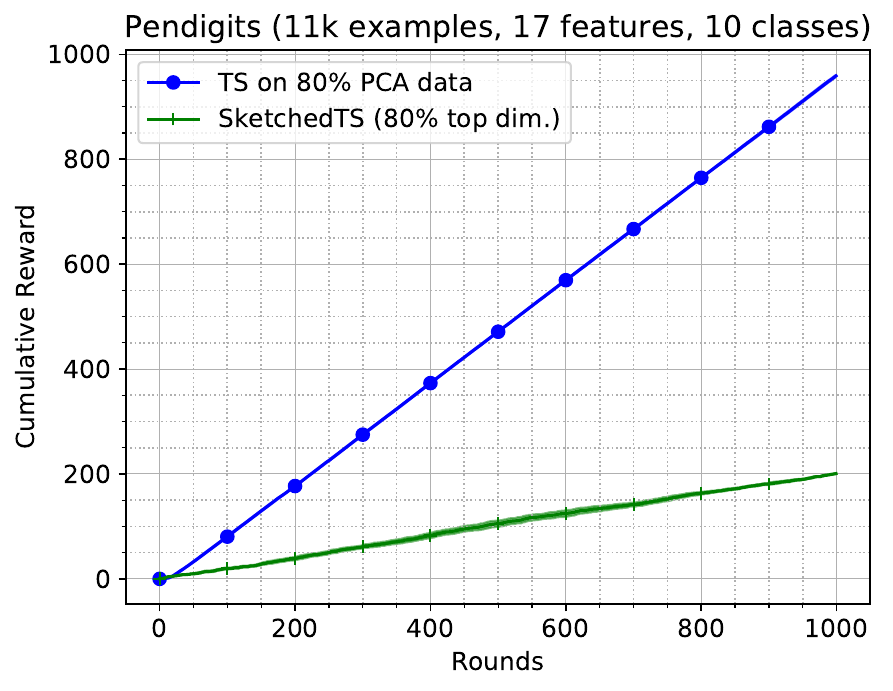}\\
  \includegraphics[width=4cm]{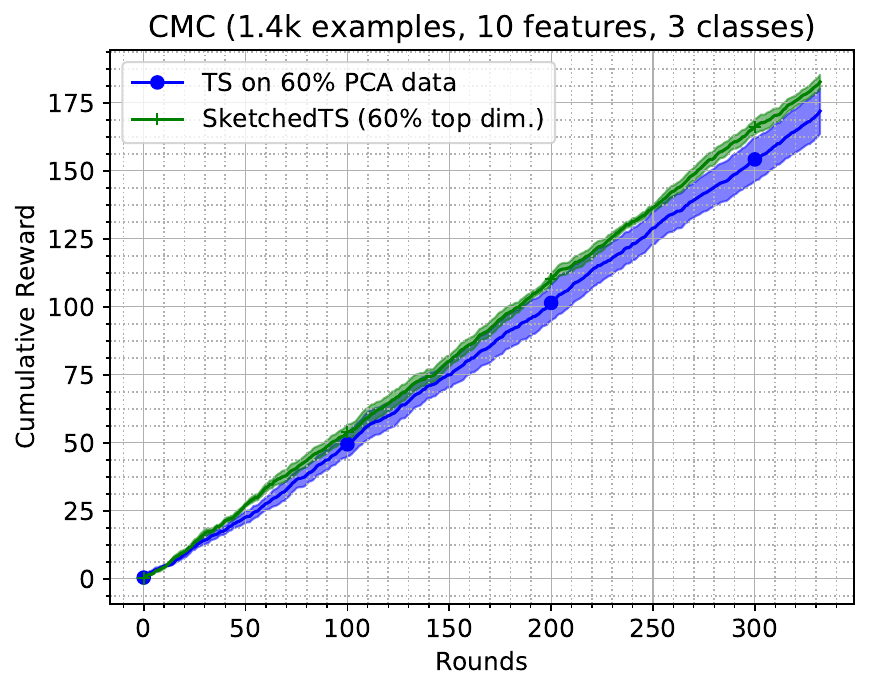}
  \includegraphics[width=4cm]{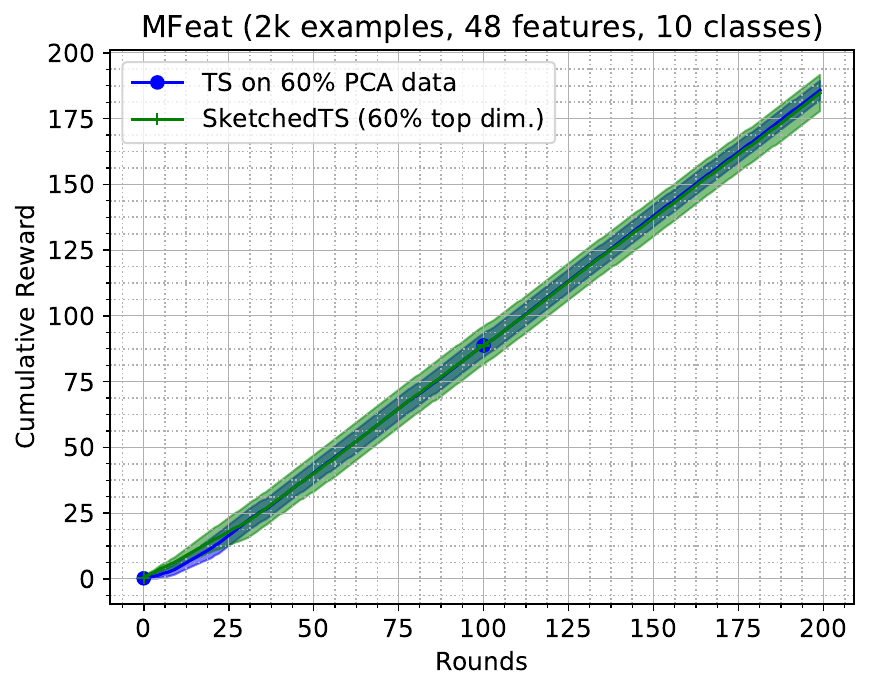}
  \includegraphics[width=4cm]{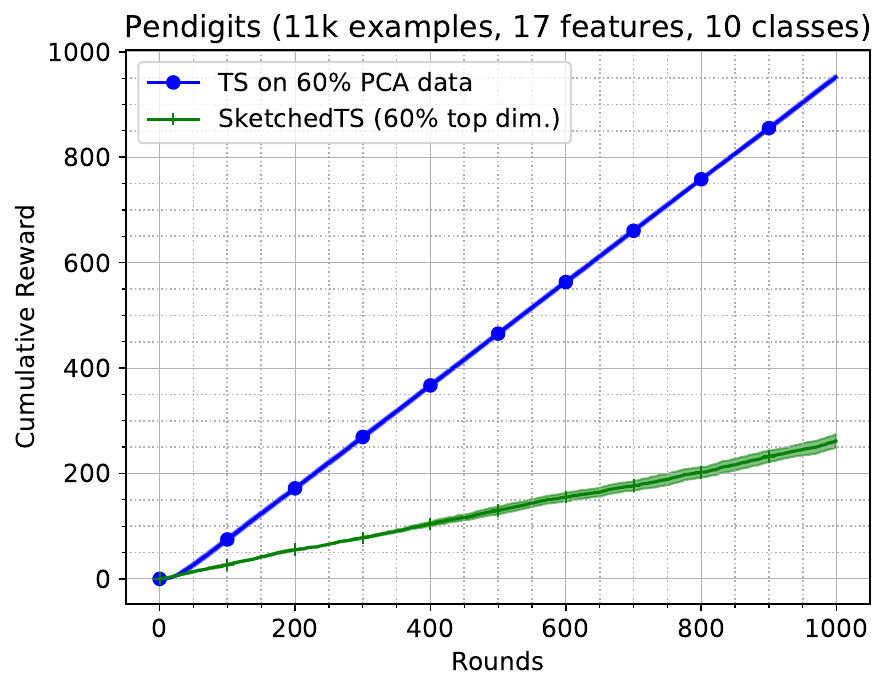}\\
  \includegraphics[width=4cm]{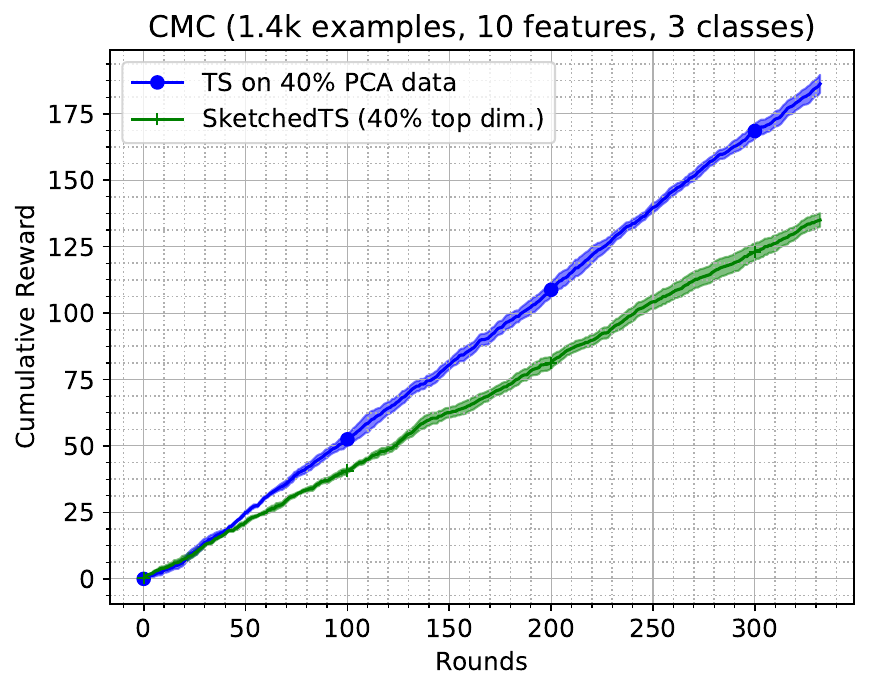}
  \includegraphics[width=4cm]{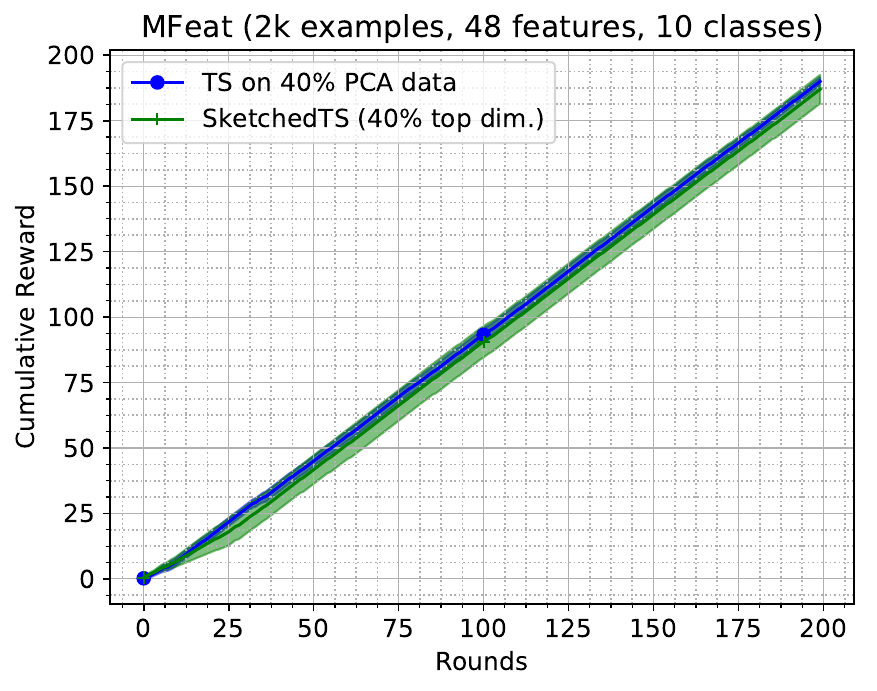}
  \includegraphics[width=4cm]{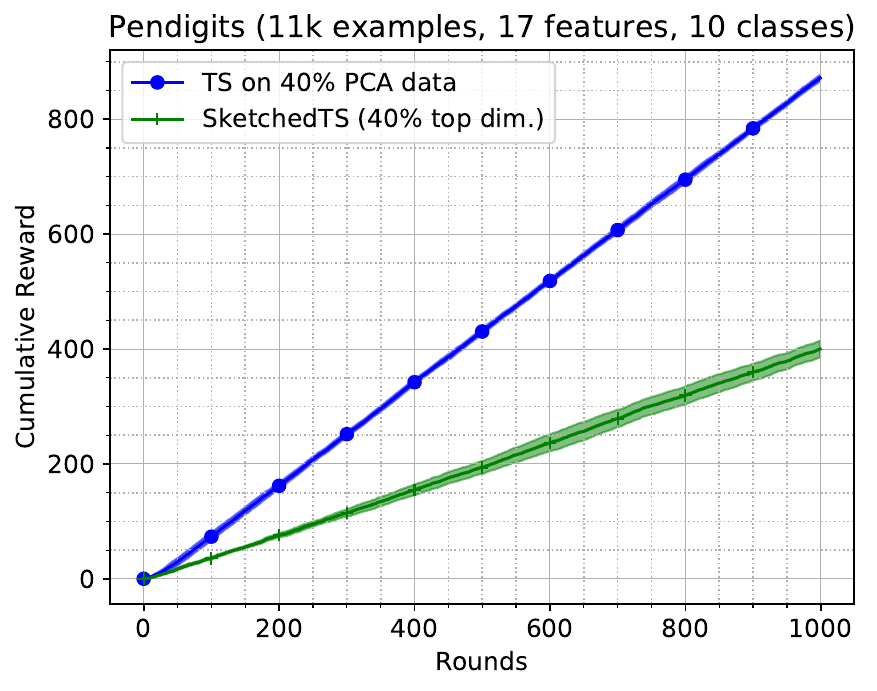}
\caption{Comparison of linear TS run on the best $m$-dimensional subspace against sketched linear TS run with sketch size $m$. Rows show $m$ as a fraction of the context space dimension: $60\%, 40\%, 20\%$ (for the first three datasets), while rows correspond to different datasets.
    Note that, in some cases (with sketch size $m$ of size at least $60\%$), sketched linear TS performs as well as if the best $m$-dimensional subspace had been known in hindsight.
}
\label{fig:appendix:pca_ts}
\end{figure*}


\clearpage
\begin{thebibliography}{16}
\providecommand{\natexlab}[1]{#1}
\providecommand{\url}[1]{\texttt{#1}}
\expandafter\ifx\csname urlstyle\endcsname\relax
  \providecommand{\doi}[1]{doi: #1}\else
  \providecommand{\doi}{doi: \begingroup \urlstyle{rm}\Url}\fi

\bibitem[Abbasi-Yadkori et~al.(2011)Abbasi-Yadkori, P{\'a}l, and
  Szepesv{\'a}ri]{abbasi2011improved}
Y.~Abbasi-Yadkori, D.~P{\'a}l, and C.~Szepesv{\'a}ri.
\newblock Improved algorithms for linear stochastic bandits.
\newblock In \emph{Advances in Neural Information Processing Systems}, pages
  2312--2320, 2011.

\bibitem[Abeille and Lazaric(2017)]{abeille2017linear}
M.~Abeille and A.~Lazaric.
\newblock Linear {T}hompson sampling revisited.
\newblock \emph{Electronic Journal of Statistics}, 11\penalty0 (2):\penalty0
  5165--5197, 2017.

\bibitem[Agrawal and Goyal(2013)]{agrawal2013thompson}
S.~Agrawal and N.~Goyal.
\newblock Thompson sampling for contextual bandits with linear payoffs.
\newblock In \emph{International Conference on Machine Learing (ICML)}, pages
  127--135, 2013.

\bibitem[Auer(2002)]{auer2002using}
P.~Auer.
\newblock Using confidence bounds for exploitation-exploration trade-offs.
\newblock \emph{Journal of Machine Learning Research}, 3\penalty0
  (Nov):\penalty0 397--422, 2002.

\bibitem[Calandriello et~al.(2017)Calandriello, Lazaric, and
  Valko]{calandriello2017efficient}
D.~Calandriello, A.~Lazaric, and M.~Valko.
\newblock Efficient second-order online kernel learning with adaptive
  embedding.
\newblock In \emph{Conference on Neural Information Processing Systems (NIPS)},
  pages 6140--6150, 2017.

\bibitem[Cesa-Bianchi et~al.(2013)Cesa-Bianchi, Gentile, and
  Zappella]{cesa2013gang}
N.~Cesa-Bianchi, C.~Gentile, and G.~Zappella.
\newblock A gang of bandits.
\newblock In \emph{Conference on Neural Information Processing Systems (NIPS)},
  pages 737--745, 2013.

\bibitem[Dani et~al.(2008)Dani, Hayes, and Kakade]{DaniHK08}
V.~Dani, T.~P. Hayes, and S.~M. Kakade.
\newblock Stochastic linear optimization under bandit feedback.
\newblock In \emph{Conference on Computational Learning Theory (COLT)}, pages
  355--366, 2008.

\bibitem[Ghashami et~al.(2016)Ghashami, Liberty, Phillips, and
  Woodruff]{ghashami2016frequent}
M.~Ghashami, E.~Liberty, J.~M. Phillips, and D.~P. Woodruff.
\newblock Frequent directions: {S}imple and deterministic matrix sketching.
\newblock \emph{SIAM Journal on Computing}, 45\penalty0 (5):\penalty0
  1762--1792, 2016.

\bibitem[Gonen et~al.(2016)Gonen, Orabona, and
  Shalev-Shwartz]{gonen2016solving}
A.~Gonen, F.~Orabona, and S.~Shalev-Shwartz.
\newblock Solving ridge regression using sketched preconditioned {SVRG}.
\newblock In \emph{International Conference on Machine Learing (ICML)}, pages
  1397--1405, 2016.

\bibitem[Higham(2008)]{higham2008functions}
N.~J. Higham.
\newblock \emph{Functions of matrices: theory and computation}, volume 104.
\newblock Siam, 2008.

\bibitem[Jun et~al.(2017)Jun, Bhargava, Nowak, and Willett]{jun2017scalable}
K.-S. Jun, A.~Bhargava, R.~Nowak, and R.~Willett.
\newblock Scalable generalized linear bandits: Online computation and hashing.
\newblock In \emph{Conference on Neural Information Processing Systems (NIPS)},
  pages 99--109, 2017.

\bibitem[Lattimore and Szepesv{\'a}ri(2018)]{lattimore2018bandit}
T.~Lattimore and C.~Szepesv{\'a}ri.
\newblock \emph{Bandit Algorithms}.
\newblock Cambridge University Press, 2018.

\bibitem[Luo et~al.(2016)Luo, Agarwal, Cesa-Bianchi, and
  Langford]{luo2016efficient}
H.~Luo, A.~Agarwal, N.~Cesa-Bianchi, and J.~Langford.
\newblock Efficient second order online learning by sketching.
\newblock In \emph{Conference on Neural Information Processing Systems (NIPS)},
  pages 902--910, 2016.

\bibitem[Vanschoren et~al.(2013)Vanschoren, van Rijn, Bischl, and
  Torgo]{openml}
J.~Vanschoren, J.~N. van Rijn, B.~Bischl, and L.~Torgo.
\newblock Open{ML}: {N}etworked {S}cience in {M}achine {L}earning.
\newblock \emph{SIGKDD Explorations}, 15\penalty0 (2):\penalty0 49--60, 2013.

\bibitem[Woodruff(2014)]{woodruff2014sketching}
D.~Woodruff.
\newblock Sketching as a tool for numerical linear algebra.
\newblock \emph{Foundations and Trends in Theoretical Computer Science},
  10\penalty0 (1--2):\penalty0 1--157, 2014.

\bibitem[Yu et~al.(2017)Yu, Lyu, and King]{yu2017cbrap}
X.~Yu, M.~R. Lyu, and I.~King.
\newblock Cbrap: Contextual bandits with random projection.
\newblock In \emph{Conference on Artificial Intelligence (AAAI)}, pages
  2859--2866, 2017.

\end{thebibliography}
\end{document}